\documentclass{article}

\usepackage{arxiv}

\usepackage[utf8]{inputenc} 
\usepackage[T1]{fontenc}    
\usepackage{hyperref}       
\usepackage{url}            
\usepackage{booktabs}       
\usepackage{amsfonts}       
\usepackage{nicefrac}       
\usepackage{microtype}      
\usepackage{lipsum}
\usepackage{graphicx}
\graphicspath{ {./images/} }
\usepackage{microtype}
\usepackage{graphicx}
\usepackage{subcaption}
\usepackage{booktabs} 
\usepackage{longtable}

\usepackage{amssymb}
\usepackage{mathtools}
\usepackage{amsthm}
\usepackage{enumitem}

\usepackage[utf8]{inputenc}
\usepackage{multirow}
\usepackage{xcolor}
\usepackage{amssymb} 
\usepackage{graphicx} 

\usepackage{amsmath,amssymb,amsthm,mathtools}
\usepackage{geometry}
\usepackage[utf8]{inputenc}
\usepackage{listings}
\usepackage{xcolor}

\newtheorem{theorem}{Theorem}

\DeclareMathOperator{\JS}{JS}
\DeclareMathOperator{\KL}{KL}
\DeclareMathOperator{\Supp}{Supp}

\title{Turing Test on Screen: 
A Benchmark for Mobile GUI Agent Humanization }

\author{
 Jiachen Zhu \\
  Shanghai Jiao Tong University\\
  Shanghai, China \\
  \texttt{gebro13@sjtu.edu.cn} \\
   \And
 Lingyu Yang \\
  Shanghai Jiao Tong University\\
  Shanghai, China \\
  \texttt{jlnhbyu.yang@sjtu.edu.cn} \\
  \And
 Rong Shan \\
  Shanghai Jiao Tong University\\
  Shanghai, China \\
  \texttt{shanrong@sjtu.edu.cn} \\
  \And
  Congmin Zheng \\
  Shanghai Jiao Tong University\\
  Shanghai, China \\
  \texttt{desp.zcm@sjtu.edu.cn} \\
  \And
  Zeyu Zheng \\
  Carnegie Mellon University\\
  Pittsburgh, Pennsylvania, USA \\
  \texttt{zeyuzhen@andrew.cmu.edu} \\
  \And
  Weiwen Liu \\
  Shanghai Jiao Tong University\\
  Shanghai, China \\
  \texttt{wwliu@sjtu.edu.cn} \\
  \And
  Yong Yu \\
  Shanghai Jiao Tong University\\
  Shanghai, China \\
  \texttt{yyu@sjtu.edu.cn} \\
  \And
  Weinan Zhang \\
  Shanghai Jiao Tong University\\
  Shanghai, China \\
  \texttt{wnzhang@sjtu.edu.cn} \\
  \And
  Jianghao Lin \\
  Shanghai Jiao Tong University\\
  Shanghai, China \\
  \texttt{linjianghao@sjtu.edu.cn} \\
}

\begin{document}
\maketitle
\begin{abstract}

The rise of autonomous GUI agents has triggered adversarial countermeasures from digital platforms, yet existing research prioritizes utility and robustness over the critical dimension of anti-detection. We argue that for agents to survive in human-centric ecosystems, they must evolve Humanization capabilities. We introduce the ``Turing Test on Screen,'' formally modeling the interaction as a MinMax optimization problem between a detector and an agent aiming to minimize behavioral divergence. We then collect a new high-fidelity dataset of mobile touch dynamics, and conduct our analysis that vanilla LMM-based agents are easily detectable due to unnatural kinematics. Consequently, we establish the Agent Humanization Benchmark (AHB) and detection metrics to quantify the trade-off between imitability and utility. Finally, we propose methods ranging from heuristic noise to data-driven behavioral matching, demonstrating that agents can achieve high imitability theoretically and empirically without sacrificing performance. This work shifts the paradigm from \textit{whether} an agent can perform a task to \textit{how} it performs it within a human-centric ecosystem, laying the groundwork for seamless coexistence in adversarial digital environments.

\end{abstract}


\section{Introduction}

The advent of Large Multimodal Models (LMMs)~\cite{gpt4, gemini, llava} has fundamentally reshaped the landscape of human-mobile interaction. By empowering systems to perceive visual interfaces and execute complex interactions, we are witnessing a paradigm shift from static scripts to autonomous Graphical User Interface (GUI) Agents~\cite{appagent, mobile_agent, cogagent}. These agents possess the capability to navigate mobile applications, process visual information, and execute tasks on behalf of users, promising a future where digital labor is significantly offloaded to AI~\cite{mind2web, yao2022webshop}.

However, the widespread deployment of GUI Agents precipitates a conflict of interest between users and service providers, potentially triggering an adversarial dynamic between autonomous agents and digital platforms~\cite{lin2025superplatformsattackaiagents,allouah2025aiagentbuyingevaluation}. As is shown in Figure~\ref{fig:intro}, modern digital ecosystems rely heavily on the \textit{attention economy}, where user engagement and advertisement impressions are the primary revenue drivers~\cite{lin2024recommendersystemsbenefitlarge}. In contrast, GUI Agents are usually optimized for efficiency and targeted for goals, bypassing promotional content and streamlining interaction paths. This behavior poses an existential threat to the business models of incumbent platforms.

This adversarial interest compels platforms to deploy Platform Defenses. These defenses may range from service blocking to more sophisticated adversarial interventions, such as injecting targeted noise or deploy advertisement traps that conversely use agents to achieve revenue goals.  As a result, these indiscriminate defenses introduce severe User Experience Risks, such as login failures or environments full of noise for real users. 
A representative example of this conflict is the recent \textit{Doubao Mobile Assistant} incident, where the agent's attempt to automate cross-application tasks triggered severe security protocols from superplatforms, such as Wechat, resulting in widespread account restrictions and service blockings. See Appendix~\ref{app:doubao_case} for details.

Despite these defensive realities, the academic community remains largely fixated on an ``Attack vs. Anti-Attack'' paradigm. Existing research predominantly focuses on two axes: (1) enhancing task utility, and (2) improving agent robustness against active platform perturbations (i.e., Anti-Attack). However, this perspective overlooks the prerequisite ``Detect vs. Anti-Detect'' paradigm. Detection acts as the gatekeeper: given the potential risks to user experience, platforms will inevitably prioritize distinguishing agents from humans to filter traffic before deploying any indiscriminate attacks.
Consequently, to achieve a harmonious coexistence with the ecosystem, agents must evolve beyond mere robustness to possess anti-detection capabilities, specifically Humanization.



\begin{figure}[t]
  \begin{center}
    \centerline{\includegraphics[width=1.0\textwidth]{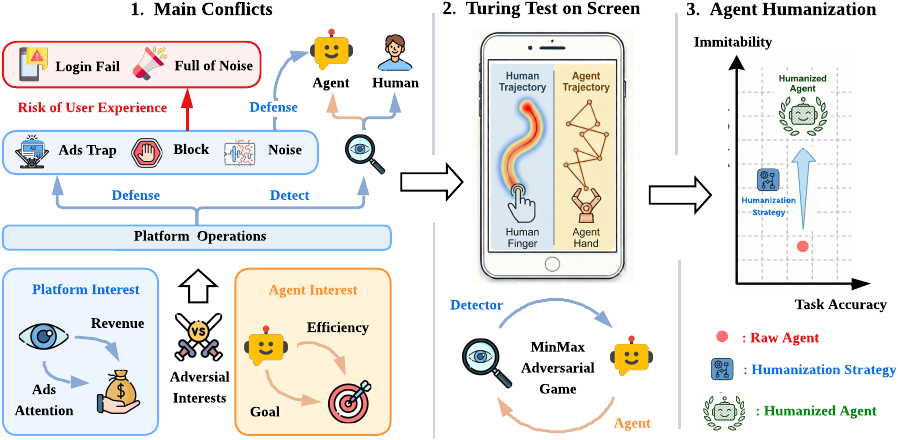}}
    \caption{The adversarial landscape between GUI Agents and Mobile Platforms. The figure illustrates three key stages: (1) \textbf{Main Conflicts:} Adversarial interests lead platforms to deploy defenses such as login blocks and ad traps. (2) \textbf{Turing Test on Screen:} The core detection mechanism relies on distinguishing natural human trajectories from agent trajectories. (3) \textbf{Agent Humanization:} We propose an adversarial humanization task to transform raw agents into humanized agents by increasing their imitability to bypass detection while maintaining task accuracy.}
    \label{fig:intro}
  \end{center}
\end{figure}

To bridge this gap, we formally define the problem of \textit{Agent Humanization} and systematically investigate the adversarial dynamics between detection and anti-detection in the era of GUI agents. We extend the Turing Test~\cite{turing1950computing} to the field of GUI Agents and introduce the concept of \textbf{``Turing Test on Screen''}.
Unlike the classical Turing Test, which evaluates human-like intelligence through textual dialogue, our paradigm evaluates human-like behavior through touch and sensor events on mobile interfaces. This draws inspiration from behavioral biometrics, where touch dynamics are traditionally used for user authentication~\cite{zaidi2021touch, frank2013touchalytics,alrawili2024comprehensivesurveybiometricuser}. In this context, the interaction is considered as an adversarial game, which is formulated as a \textbf{MinMax optimization problem}~\cite{von1944theory,goodfellow2014generative} where the Detector seeks to maximize the distinction between human and agent behaviors, while the GUI Agent seeks to minimize this \textbf{distinction} without decreasing \textbf{task utility}.

Guided by this formulation, we conduct a comprehensive study to assess the current state of agent detectability. We collect a large-scale dataset comprising detailed motion events such as touch coordinates, velocity and sensor events from both human users and a wide range of state-of-the-art GUI Agents. Our empirical analysis reveals that raw agents are highly susceptible to detection due to unnatural kinematic features. Based on these findings, we construct the Agent Humanization Benchmark (AHB) to evaluate the trade-off between human-like \textbf{imitability} and task success \textbf{utility}. Furthermore, we propose multiple humanization strategies designed to evade detection, conducting both \textbf{theoretical proofs} in Section~\ref{sec:methodology} and Appendix~\ref{app:theory} and \textbf{empirical experiments} in Section~\ref{sec:experiments} to prove their effectiveness, providing a roadmap for future agent development.

The contributions of this paper are summarized as follows:

\begin{itemize}[leftmargin=10pt]
    \item  We are the first to extend the Turing Test to the field of GUI Agents and introduce
the concept of ``Turing Test on Screen''. We formally define the adversarial paradigm between the Detector and the GUI Agent, establishing a theoretical framework for studying agent detectability in GUI environments.
    \item We construct a rich dataset containing granular \texttt{MotionEvent} and \texttt{SensorEvent} sequences, enabling high-fidelity analysis of behavioral differences between humans and GUI agents.
    \item We are the first to propose specific detection metrics and establish the Agent Humanization Benchmark (AHB) to quantitatively assess agent imitability and utility.
    \item We design and evaluate several humanity modules, ranging from heuristic noise injection to data-driven history matching, which improve agent imitability both theoretically and empirically. Our code and data are publicly available at \footnote{https://github.com/Gebro13/Passing-the-Turing-Test-on-Screen-Agent-Humanization-Benchmark} and \footnote{https://huggingface.co/datasets/lyyang2766/Passing-the-Turing-Test-on-Screen-Agent-Humanization-Benchmark/tree/main}.
\end{itemize}

Ultimately, this work underscores a pivotal transition in the evolution of AI agents: moving beyond the question of \textit{whether} an agent can perform a task, to \textit{how} it performs it within a human-centric ecosystem. As the ``Turing Test on Screen'' becomes inevitable for digital access, the ability to exhibit human-like behavioral nuances is no longer merely an aesthetic feature but a functional necessity for survival. By formalizing the interplay between detection and humanization, we hope to lay the groundwork for a future where autonomous agents can seamlessly coexist with existing digital infrastructures, safeguarding user agency in an increasingly adversarial online world.

\section{Formulation of Turing Test on Screen}
\label{sec:problem_formulation}


We formally define the \textbf{Turing Test on Screen} as a Min-Max adversarial game~\cite{von1944theory} between two entities: a \textbf{Detector} $D_{\Theta}$ (the platform) and a \textbf{GUI Agent} $G_{\Phi}$ (the operator). The detector aims to maximize classification accuracy, while the agent minimizes detection probability subject to task utility constraints.

\subsection{Interaction Modeling}

The interaction between the agent and the mobile OS occurs at two distinct layers: the logical \textit{action} level and the physical \textit{event} level.
Agent-OS interaction is decoupled into two hierarchical layers:

\begin{description}[leftmargin=5pt]
    \item[Agent Level:] At each step $t$, $G_{\Phi}$ generates a high-level UI command $a_t$ (e.g., tap, swipe) based on the environmental state $s_t$:
    \begin{equation}
        a_t = G_{\Phi}(s_t), \quad s_{t+1} = \mathcal{T}(s_t, a_t)
    \end{equation}
    where $\mathcal{T}$ denotes the state transition function.
    
    \item[Event Level:]On a mobile phone, a single logical action $a_t$ is not a simple data point; rather, it acts as a trigger that invokes multiple underlying hardware sensors, generating a set of fine-grained events $E_t$. We define the event mapping function $f: a \to \{e\}$ such that:
\begin{equation}
    E_t = \{e_{t,1}, e_{t,2}, \dots, e_{t,k}\} = f(a_t)
\end{equation}
 These events $e \in E_t$ are categorized as: (1) \textbf{Motion Events ($M$)} representing touch dynamics (coordinates, pressure); and (2) \textbf{Sensor Events ($S$)} representing physical signals (gyroscope, magnetometer). 
 
Thus, $e \in M \cup S$. The complete behavioral trace observed by the system up to time $T$ is the union of all triggered events: $\mathcal{E}_{1:T} = \bigcup_{t=1}^{T} E_t$.

\end{description}

\subsection{The Adversarial Game}
The benchmark evaluates whether the event sequence $\mathcal{E}$ is distinguishable from human-generated patterns.

\paragraph{Detector's Objective}
$D_{\Theta}$ acts as a discriminator evaluating the accumulated stream $\mathcal{E}_{1:t}$. For any action sequence, it outputs a probability $y_t = D_{\Theta}(\mathcal{E}_{1:t}) \in [0, 1]$, where $y_t \to 1$ denotes a Human classification and 0 denotes an agent. $D_{\Theta}$ maximizes its discrimination power:
\begin{equation}
    \max_{\Theta} \mathcal{L}_D = \mathbb{E}_{\mathcal{E} \sim \mathcal{H}} [\log D_{\Theta}(\mathcal{E})] + \mathbb{E}_{\mathcal{E} \sim G_{\Phi}} [\log (1 - D_{\Theta}(\mathcal{E}))]
\end{equation}
where $\mathcal{H}$ and $G_{\Phi}$ denote the event distributions of humans and agents, respectively.

\paragraph{Agent's Objective}
$G_{\Phi}$ must optimize its parameters $\Phi$ to balance \textit{Imitability} and \textit{Utility} via a regularized minimization:
\begin{equation}
    \min_{\Phi} \mathcal{L}_G = \mathbb{E}_{s \sim \mathcal{S}} \left[ \sum_{t=1}^{T} \mathbb{I}(D_{\Theta}(\mathcal{E}_{1:t}) < \tau) - \lambda \cdot R_{\text{task}}(G_{\Phi}) \right]
\end{equation}
where $\tau$ is the detection threshold, $\mathbb{I}(\cdot)$ is the indicator function, and $R_{\text{task}}$ represents the task success rate. The multiplier $\lambda$ governs the trade-off, ensuring humanization does not compromise functional capability. This framework provides the theoretical foundation for the \textit{Agent Humanization Benchmark} (AHB).

\section{Data Collection and Preliminary Study}
\label{sec:data_analysis}

We first collect a large-scale data and make some preliminary studies. This stage focuses on understanding the behavioral signatures of standard GUI agents compared to authentic human users.

\subsection{Dataset Collection} 
Our dataset captures interactions across 21 diverse applications categorized into five clusters (e.g., Social Media, Shopping, Trip Planning; see Table~\ref{tab:clusters and applications} in Appendix). Data is collected from two primary sources:
\begin{itemize}[leftmargin=10pt]
    \item \textbf{Human Users:} Four sub-populations (Young Man/Woman, Middle-aged, and Elderly) to capture physiological and age-related behavioral variances.
    \item \textbf{Autonomous Agents:} Interactions generated by state-of-the-art models including UI-TARS~\cite{qin2025uitarspioneeringautomatedgui}, MobileAgent-E (GPT-4o)~\cite{wang2025mobileagenteselfevolvingmobileassistant}, MobileAgent-E (Claude-3.5-Sonnet), AgentCPM~\cite{zhang2025agentcpmguibuildingmobileuseagents}, and AutoGLM~\cite{liu2024autoglmautonomousfoundationagents}.
\end{itemize}


Following~\cite{frank2012touchalytics}, we derive \textbf{24 statistical features} to capture unique biomechanical signatures. These include \textit{Kinematics}: e.g. velocity, acceleration, \textit{Geometry}: e.g. path efficiency, curvature, and \textit{Temporal Dynamics}: e.g. duration, latency. 
To quantify the relevance of the feature and to prove its effectiveness for the detector, we calculate the \textbf{Information Gain (IG)}~\cite{shannon1948mathematical} for each attribute relative to the source identity $U$. See details in Table~\ref{tab:touch_features} and Figure~\ref{fig:touch_features} in Appendix~\ref{sec:feature_extraction}. The strategic decision to focus on touch dynamics rather than hardware sensor streams is justified in Appendix~\ref{app:discussion}.


\begin{figure*}[t]
    \centering
    \vspace{-10pt}
    \begin{subfigure}[ht]{0.44\textwidth}
        \centering

        \includegraphics[width=1.0\textwidth]{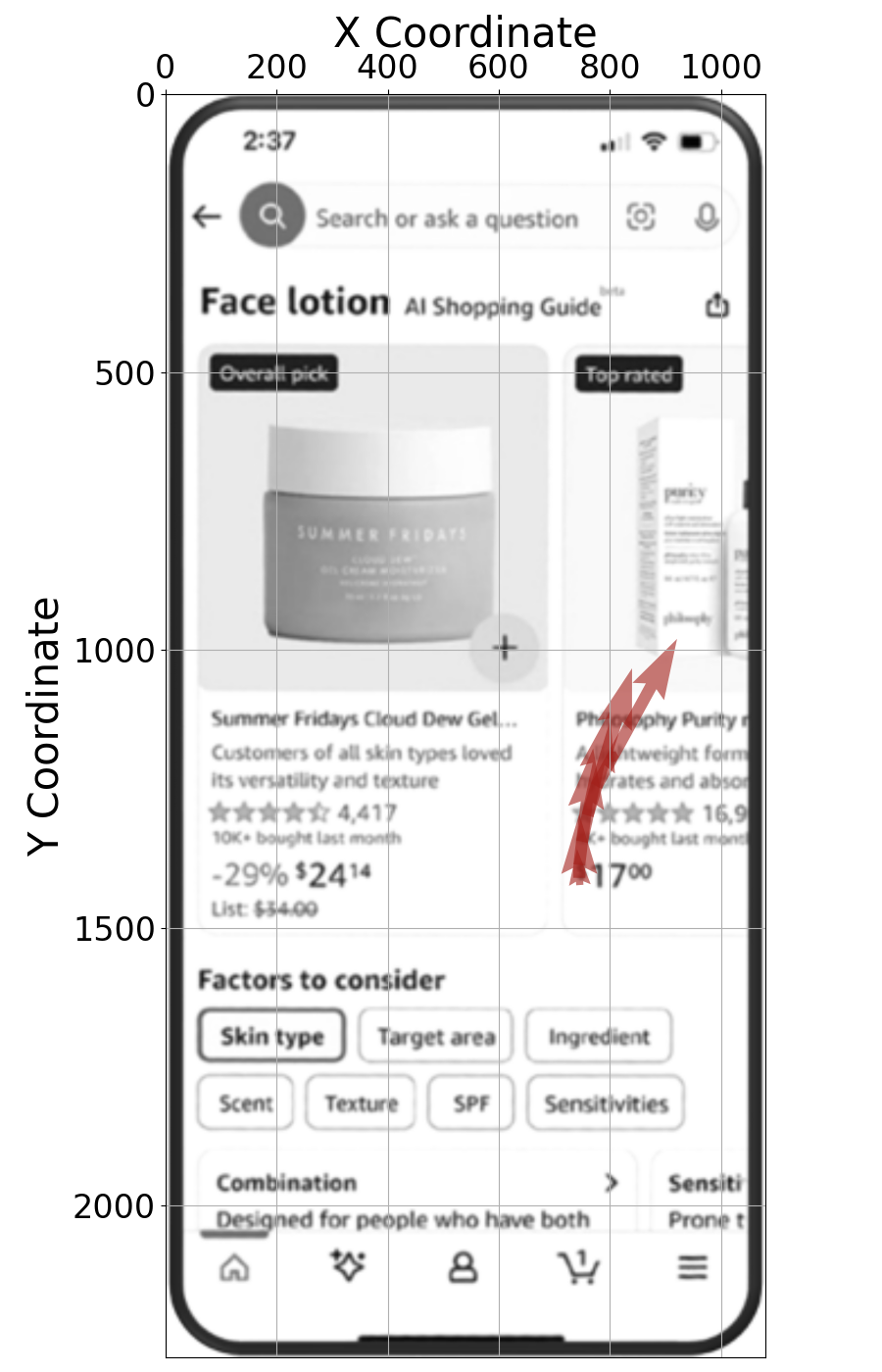}
        \caption{Human Swipe}
        \label{fig:human swipe}
    \end{subfigure}
    \hfill 
    \begin{subfigure}[ht]{0.44\textwidth}
        \centering
        \includegraphics[width=1.0\textwidth]{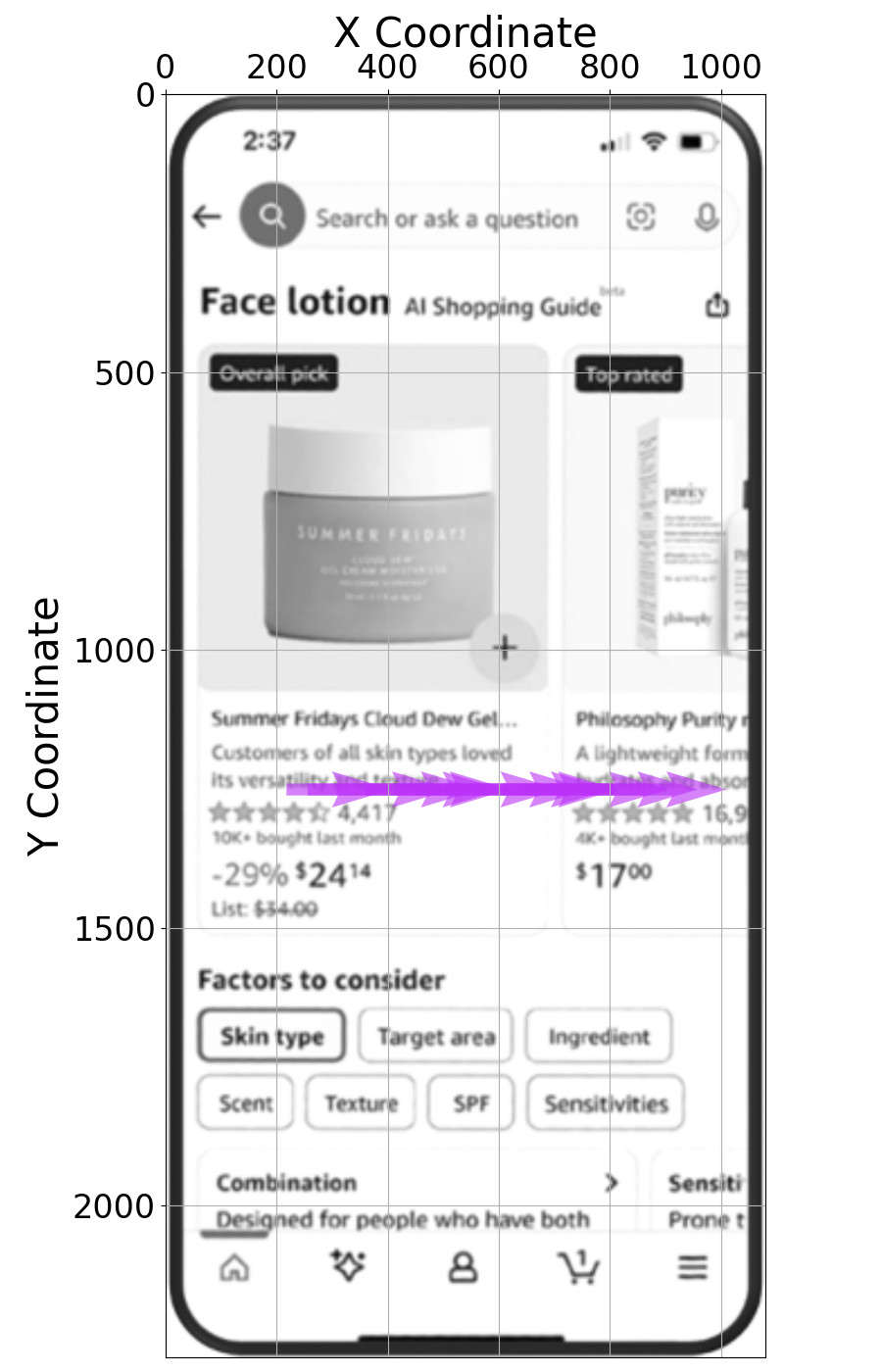}
        
        \caption{Agent Swipe}
        \label{fig: agent swipe}
    \end{subfigure}
    \caption{The difference between human and agent swipe.}
    \label{fig: human and agent swipe}
\end{figure*}



\begin{figure}[t]
    \centering
    \vspace{-10pt}
    \begin{subfigure}[b]{0.44\textwidth}
        \centering

        \includegraphics[width=1.0\textwidth]{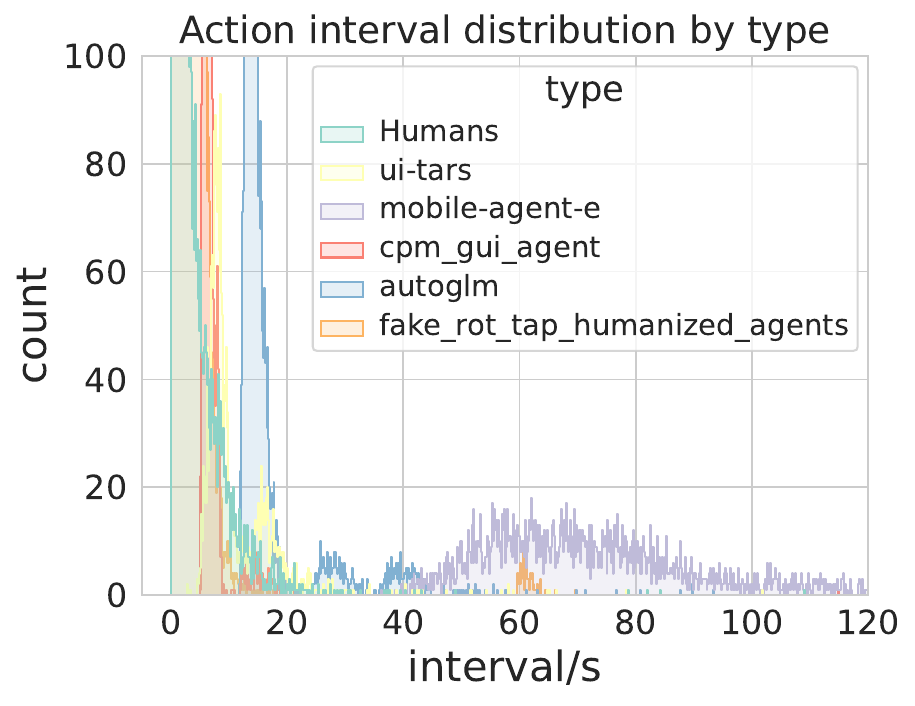}
        \caption{Action Interval}
        \label{fig:action interval}
    \end{subfigure}
    \hfill 
    \begin{subfigure}[b]{0.44\textwidth}
        \centering
        \includegraphics[width=1.0\textwidth]{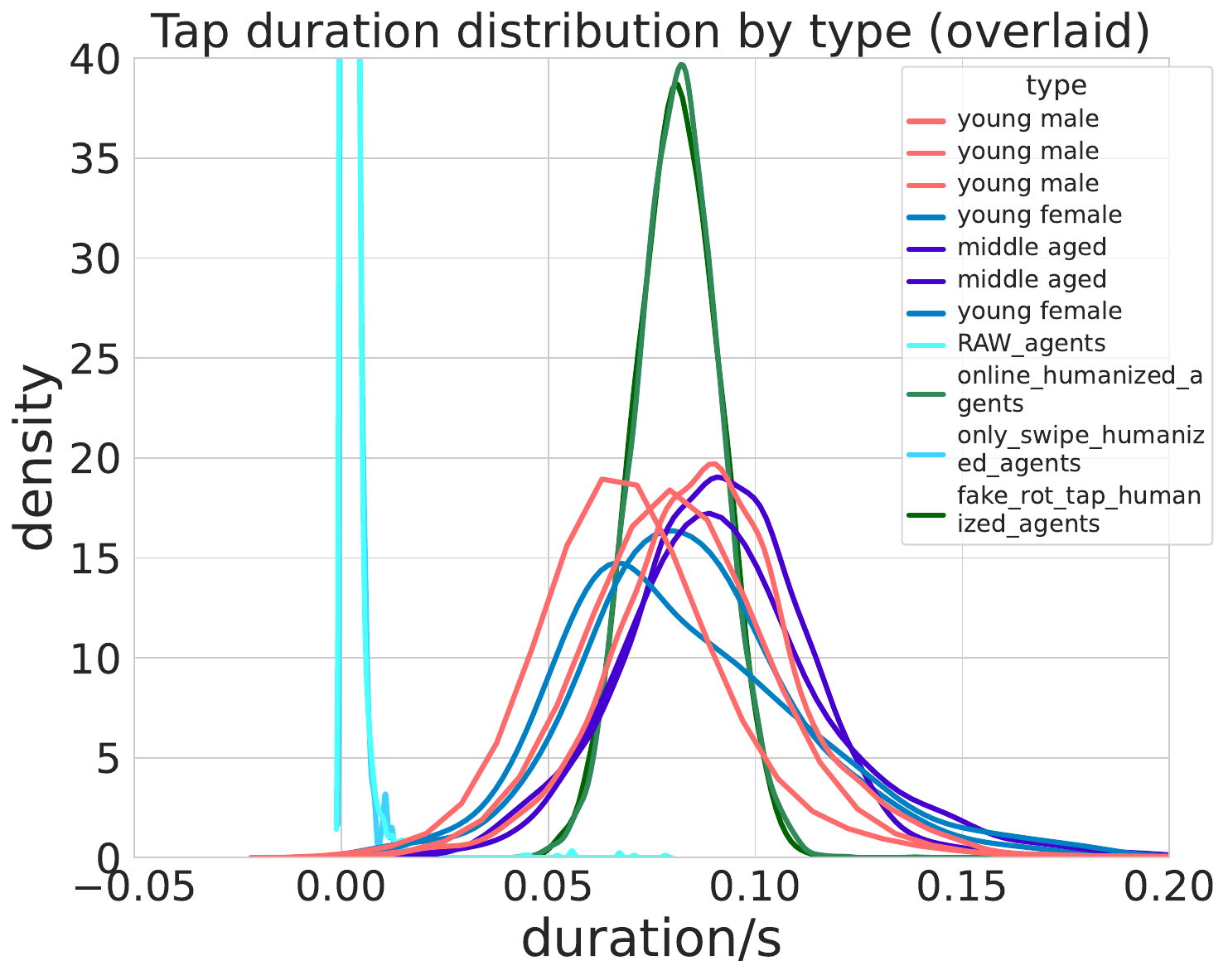}
        
        \caption{Tap Duration}
        \label{fig:tap duration}
    \end{subfigure}
    \caption{The visualization of action interval and tap duration differences between human and agents.}
    \label{fig:main_figure}
\end{figure}

\subsection{Preliminary Qualitative Study}
We conduct a preliminary  study and demonstrate the differences across two aspects, as follows:

\subsubsection{Trajectory Linearity}
As shown in Figure~\ref{fig: human and agent swipe}, agent trajectories are typically rigid, linear vectors lacking the physiological arcs and motor noise of human gestures. 

\subsubsection{Action Intervals.} Human intervals follow a long-tailed distribution peaking near zero, whereas agents suffer from significant inference overhead. As shown in Figure~\ref{fig:action interval}, \texttt{ui-tars} clusters at $5$--$10$s, while \texttt{mobile-agent-e} reaches $50$--$80$s. These delays are sufficient for reliable detection.

\subsubsection{Tap Duration.} Human taps form a Gaussian distribution ($0.05$s--$0.10$s) due to skin elasticity, while agent inputs manifest as near-zero spikes, reflecting instantaneous event injection.

In summary, vanilla agents fail the ``Turing Test on Screen'' due to robotic linearity and non-human temporal rhythms, making it trivial for a detector to identify.

\section{The Agent Humanization Benchmark}
\label{sec:benchmark}
To quantify humanization effectiveness, we introduce the \textbf{Agent Humanization Benchmark (AHB)}, a framework evaluating agents across two axes: \textbf{Imitability} and \textbf{Utility}.

\subsection{Evaluation Metrics}

\subsubsection{Imitability}
Imitability measures the behavioral resemblance between agents and humans, quantified inversely by the \textbf{Classification Accuracy (ACC)} of various detection algorithms. A detector ACC approaching $0.5$ (random guessing) signifies that the agent has successfully passed the ``Turing Test on Screen'' for that detection modality.

\subsubsection{Utility}
Since humanization (e.g., noise, delays) may degrade efficiency, we monitor the \textbf{Task Success Rate} to ensure functionality is preserved. An ideal strategy achieves high imitability with minimal success rate degradation; strategies that bypass detection but fail at tasks are considered unsuccessful.

\subsection{The Hierarchy of Detectors}
The AHB categorizes detectors ($D$) by action type and complexity, ranging from simple heuristics to robust machine learning models to assess agents against a defense hierarchy.

\subsubsection{Rule-based Detectors}
These serve as the first line of defense, utilizing predefined statistical thresholds to identify anomalies in individual attributes. Metrics include \textbf{Swipe Accuracy}, \textbf{Time Interval Accuracy}, and \textbf{Tap Duration Accuracy}. Together, they filter out agents that fail to adhere to basic biological constraints.

\subsubsection{Learning-based Detectors}
To identify subtle, non-linear patterns in trajectories, we employ SVM~\cite{cortes1995support} and XGBoost~\cite{chen2016xgboost} classifiers. Trained on the 24-dimensional feature vector in Section~\ref{sec:data_analysis}, these models capture complex correlations between features. Evading these detectors requires the agent to mimic the holistic distribution of human behavior rather than just isolated features.
A deeper discussion on the robustness of these interpretable detectors is provided in Appendix~\ref{app:discussion}.
\subsection{Humanization Methodologies}
\label{sec:methodology}

To bridge the gap between mechanical agent behavior and human-like interaction, we propose two architectural paradigms for humanization: (1) \textbf{Internal Injection}, embedding human priors directly into the LMM (e.g., via fine-tuning); and (2) \textbf{External Wrapper}, a post-processing module $H$ that transforms raw actions $a_{\text{raw}}$ into humanized sequences $a_{\text{human}}$ before execution. 

As a foundational study, we adopt the \textbf{External Wrapper} approach for its model-agnostic compatibility. We propose four distinct strategies, starting with heuristic signal processing.

\subsubsection{Strategy 1: Heuristic Noise Injection (B-Spline)}
To counteract the perfect linearity of raw agent swipes, we utilize B-spline smoothing~\cite{de1972calculating}. Instead of a linear path, we generate a curve $S(t)$ based on control points $C = \{c_0, \dots, c_n\}$ scattered normally around the direct chord:
\begin{equation}
    S(t) = \sum_{i=0}^{n} N_{i,p}(t) \cdot c_i
\end{equation}
where $N_{i,p}(t)$ are B-spline basis functions of degree $p$. While computationally efficient and real-time capable, this method may remains statistically distinguishable if the noise distribution does not precisely match human biomechanical curvature.

\subsubsection{Strategy 2: Data-Driven History Matching}
To achieve higher fidelity, we leverage real human trajectories from our dataset. Given a task vector $\vec{v}_{\text{task}}$, we sample a reference trajectory $\mathcal{T}_{\text{ref}}$ with similar distance and direction, then apply an affine transformation to align it. Each point $p \in \mathcal{T}_{\text{ref}}$ is transformed to $p'$ via:
\begin{equation}
    p' = s \cdot R(\theta) \cdot (p - p_{\text{ref\_start}}) + P_{\text{start}}
\end{equation}
where $R(\theta)$ is a rotation matrix based on the angular difference, and $s = \|\vec{v}_{\text{task}}\| / \|\vec{v}_{\text{ref}}\|$ is the scaling factor. 

This strategy preserves authentic velocity profiles and micro-jitters, though it requires an offline database.

\subsubsection{Theoretical Foundations}
We provide formal proofs for three theorems in Appendix~\ref{app:theory}. Theorem~\ref{thm:opt_detector} bounds a detector's efficacy by the Jensen-Shannon divergence between human and agent distributions. Theorem~\ref{thm:smoothing} proves that variance injection (e.g., B-spline) strictly reduces this divergence. Finally, Theorem~\ref{thm:history_matching} demonstrates that History Matching is asymptotically superior, as agent behavior converges toward the true human distribution.

\subsubsection{Strategy 3: Fake Actions}
To mask the long inference latencies identified in Section~\ref{sec:data_analysis}, the wrapper injects micro-interactions (e.g., slight scrolls or hovers) during idle periods. These non-functional inputs break the long-tail interval distribution, shifting the agent's temporal profile toward continuous human-like interaction.

\subsubsection{Strategy 4: Longer Presses}
To humanize the near-zero tap durations of raw agents, we sample durations from a Gaussian distribution fitted to human tap data, ensuring touch events mimic realistic physical contact.
\section{Experiments \& Analysis}
\label{sec:experiments}

In this section, we provide our whole detection and humanization results, as well as in-depth feature analysis to find the easiest and hardest feature to humanize.

\begin{table*}[!ht]
\centering
\footnotesize

\caption{Experiment results of humanization strategies across five distinct application domains. We compare the baseline (RAW) against various combinations of humanization methods, including swipe trajectory adjustment (B-spline vs. History Matching), interval noise injection (Fake), and tap duration adjustment (Long). The right block reports the detection accuracy (lower is better) of different classifiers SVM and XGBoost and rule-based checks, alongside the final Task Accuracy.}
\resizebox{\textwidth}{!}{%
\begin{tabular}{llcccccccccc}

\toprule
\multirow{3}{*}{Task} & \multirow{3}{*}{Mode} & \multicolumn{4}{c}{Humanization Methods} & \multicolumn{6}{c}{Detection Rules} \\
\cmidrule(lr){3-6} \cmidrule(lr){7-12}
& & \multicolumn{2}{c}{swipe} & interval & tap & \multicolumn{3}{c}{swipe} & interval & tap & utility \\
\cmidrule(lr){3-4} \cmidrule(lr){5-5} \cmidrule(lr){6-6} \cmidrule(lr){7-9} \cmidrule(lr){10-10} \cmidrule(lr){11-11} \cmidrule(lr){12-12}
& & b-spline & history & fake & long & max single & SVM acc & XGB acc & int.acc & tap.acc & task acc \\
\midrule

\multirow{10}{*}{Social Media} 
& RAW & x & x & x & x & 0.9969 & 0.9817 & 1.0000 & 0.8838 & 0.9977 & 0.4833 \\
& online & x & \checkmark & x & \checkmark & 0.8286 & 0.8750 & 0.9773 & 0.8798 & 0.6341 & 0.5625 \\
& online & x & \checkmark & x & x & 0.7651 & 0.9756 & 1.0000 & 0.9060 & 0.9976 & 0.6667 \\
& online & x & \checkmark & \checkmark & \checkmark & 0.9998 & 0.9963 & 0.9993 & 0.5999 & 0.6210 & 0.4500 \\
& offline & x & \checkmark & x & x & 0.7190 & 0.9633 & 0.9450 & 0.8838 & 0.9977 & - \\
& offline & \checkmark & x & x & x & 0.8507 & 0.9633 & 0.9817 & 0.8838 & 0.9977 & - \\
& offline & x & x & x & \checkmark & 1.0000 & 0.9773 & 1.0000 & 0.8798 & 0.6341 & - \\
& offline & x & x & \checkmark & x & 0.9969 & 0.9817 & 1.0000 & 0.5274 & 0.9977 & - \\
& offline & x & \checkmark & \checkmark & \checkmark & 0.8286 & 0.8750 & 0.9773 & 0.5260 & 0.6341 & - \\
& offline & \checkmark & x & \checkmark & \checkmark & 0.8507 & 0.9633 & 0.9817 & 0.5274 & 0.6137 & - \\
\midrule

\multirow{10}{*}{Shopping}
& RAW & x & x & x & x & 0.9982 & 0.9887 & 1.0000 & 0.9056 & 0.9840 & 0.8148 \\
& online & x & \checkmark & x & \checkmark & 0.9249 & 0.9593 & 0.9889 & 0.8969 & 0.6133 & 0.7069 \\
& online & x & \checkmark & x & x & 0.9769 & 0.9570 & 0.9785 & 0.9196 & 0.9971 & 0.9500 \\
& online & x & \checkmark & \checkmark & \checkmark & 0.9989 & 0.9962 & 0.9986 & 0.5718 & 0.6278 & 0.6000 \\
& offline & x & \checkmark & x & x & 0.8780 & 0.9323 & 0.9925 & 0.9056 & 0.9840 & - \\
& offline & \checkmark & x & x & x & 0.9336 & 0.9774 & 0.9925 & 0.9056 & 0.9840 & - \\
& offline & x & x & x & \checkmark & 1.0000 & 0.9778 & 1.0000 & 0.8969 & 0.6133 & - \\
& offline & x & x & \checkmark & x & 0.9982 & 0.9887 & 1.0000 & 0.5089 & 0.9840 & - \\
& offline & x & \checkmark & \checkmark & \checkmark & 0.9249 & 0.9593 & 0.9889 & 0.5104 & 0.6133 & - \\
& offline & \checkmark & x & \checkmark & \checkmark & 0.9336 & 0.9774 & 0.9925 & 0.5089 & 0.6105 & - \\
\midrule
\multirow{10}{*}{Video Streaming}
& RAW & x & x & x & x & 1.0000 & 0.9850 & 1.0000 & 0.9186 & 0.9956 & 0.6094 \\
& online & x & \checkmark & x & \checkmark & 0.9494 & 0.9502 & 0.9950 & 0.9120 & 0.6186 & 0.8393 \\
& online & x & \checkmark & x & x & 0.9929 & 0.9942 & 0.9770 & 0.9112 & 0.9990 & 0.7500 \\
& online & x & \checkmark & \checkmark & \checkmark & 0.9993 & 0.9968 & 0.9974 & 0.5621 & 0.6276 & 0.7500 \\
& offline & x & \checkmark & x & x & 0.9306 & 0.9300 & 0.9850 & 0.9186 & 0.9956 & - \\
& offline & \checkmark & x & x & x & 0.9390 & 0.9650 & 0.9850 & 0.9186 & 0.9956 & - \\
& offline & x & x & x & \checkmark & 1.0000 & 0.9950 & 1.0000 & 0.9120 & 0.6186 & - \\
& offline & x & x & \checkmark & x & 1.0000 & 0.9850 & 1.0000 & 0.5195 & 0.9956 & - \\
& offline & x & \checkmark & \checkmark & \checkmark & 0.9494 & 0.9502 & 0.9950 & 0.5196 & 0.6186 & - \\
& offline & \checkmark & x & \checkmark & \checkmark & 0.9390 & 0.9650 & 0.9850 & 0.5195 & 0.6129 & - \\
\midrule
\multirow{10}{*}{Trip Planning}
& RAW & x & x & x & x & 0.9984 & 0.9817 & 0.9954 & 0.7998 & 0.9954 & 0.7500 \\
& online & x & \checkmark & x & \checkmark & 0.8153 & 0.9479 & 0.9905 & 0.8047 & 0.6264 & 0.7143 \\
& online & x & \checkmark & x & x & 0.8721 & 0.9278 & 0.9896 & 0.8640 & 0.9981 & 0.7000 \\
& online & x & \checkmark & \checkmark & \checkmark & 0.9992 & 0.9945 & 0.9989 & 0.5011 & 0.6110 & 0.1500 \\
& offline & x & \checkmark & x & x & 0.8421 & 0.8995 & 0.9863 & 0.7998 & 0.9954 & - \\
& offline & \checkmark & x & x & x & 0.8855 & 0.9726 & 0.9909 & 0.7998 & 0.9954 & - \\
& offline & x & x & x & \checkmark & 0.9970 & 0.9953 & 1.0000 & 0.8047 & 0.6264 & - \\
& offline & x & x & \checkmark & x & 0.9984 & 0.9817 & 0.9954 & 0.5718 & 0.9954 & - \\
& offline & x & \checkmark & \checkmark & \checkmark & 0.8153 & 0.9479 & 0.9905 & 0.5704 & 0.6264 & - \\
& offline & \checkmark & x & \checkmark & \checkmark & 0.8855 & 0.9726 & 0.9909 & 0.5718 & 0.6089 & - \\
\midrule
\multirow{10}{*}{Office \& Learning}
& RAW & x & x & x & x & 1.0000 & 0.9826 & 1.0000 & 0.8744 & 0.9974 & 0.5750 \\
& online & x & \checkmark & x & \checkmark & 0.7782 & 0.9265 & 0.9926 & 0.8803 & 0.6263 & 0.5167 \\
& online & x & \checkmark & x & x & 0.9446 & 0.9571 & 1.0000 & 0.8967 & 0.9988 & 0.5375 \\
& online & x & \checkmark & \checkmark & \checkmark & 0.9997 & 0.9953 & 0.9989 & 0.5606 & 0.6177 & 0.2875 \\
& offline & x & \checkmark & x & x & 0.7720 & 0.9391 & 0.9739 & 0.8744 & 0.9974 & - \\
& offline & \checkmark & x & x & x & 0.8178 & 0.9739 & 0.9913 & 0.8744 & 0.9974 & - \\
& offline & x & x & x & \checkmark & 1.0000 & 0.9926 & 0.9926 & 0.8803 & 0.6263 & - \\
& offline & x & x & \checkmark & x & 1.0000 & 0.9826 & 1.0000 & 0.5183 & 0.9974 & - \\
& offline & x & \checkmark & \checkmark & \checkmark & 0.7782 & 0.9265 & 0.9926 & 0.5182 & 0.6263 & - \\
& offline & \checkmark & x & \checkmark & \checkmark & 0.8178 & 0.9739 & 0.9913 & 0.5183 & 0.6134 & - \\

\bottomrule
\end{tabular}%
}
\vspace{-10pt}
\label{tab:main table}
\end{table*}

\subsection{Comprehensive Results of Humanization}

\subsubsection{Baseline Vulnerability Analysis}
Results in Table~\ref{tab:main table} confirm the high detectability of non-humanized agents. Across all clusters, the XGBoost and SVM classifiers achieve near-perfect accuracy (e.g., $0.995$ and $0.98$), demonstrating that raw agent trajectories contain distinct mechanical patterns that are trivial for ML models to identify.

\subsubsection{Experimental Settings}
We implement four humanization strategies targeting distinct dimensions:
\begin{itemize}[leftmargin=10pt]
    \item \textbf{Swipe:} \textit{B-spline Noise} and \textit{History Matching} serve as mutually exclusive trajectory generation methods.
    \item \textbf{Tap \& Interval:} \textit{Long Press} (tap duration) and \textit{Fake Action} (temporal intervals) can be superimposed on any swipe strategy.
\end{itemize}
Evaluations are conducted in \textbf{online} mode: real-time execution; measures task success and \textbf{offline} mode: post-hoc modification; isolates detection evasion from static utility.

\subsubsection{Main Results}

\begin{figure*}[ht]
  \vspace{-10pt}
  \begin{center}
    \centerline{\includegraphics[width=1.05\textwidth]{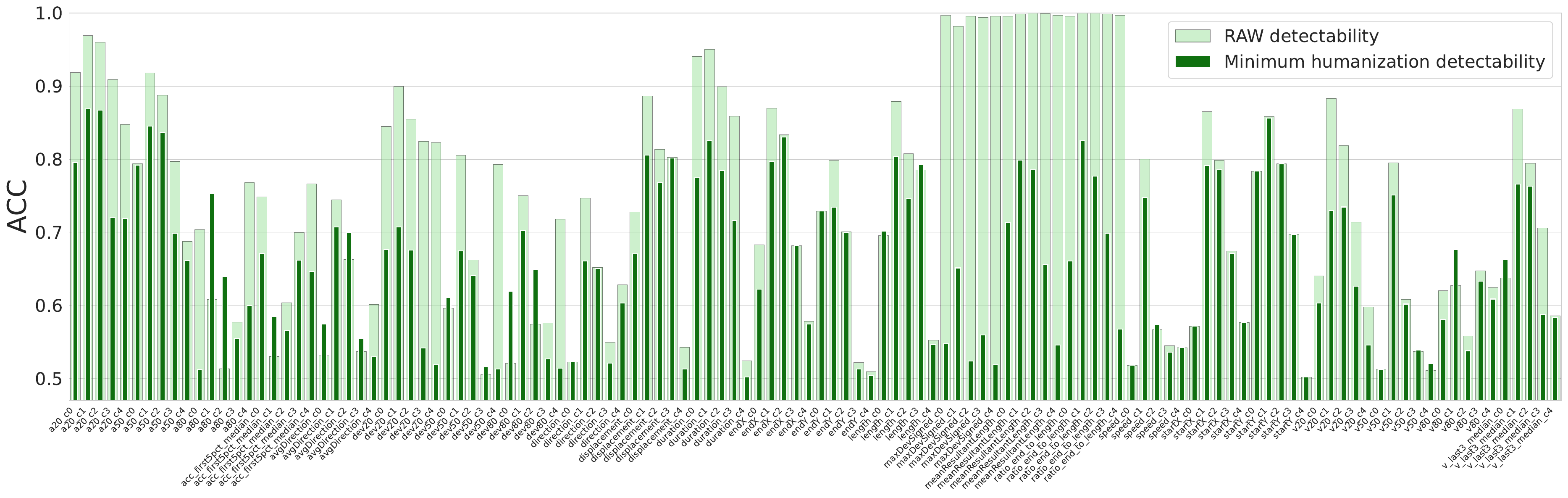}}
    \vspace{-10pt}
    \caption{Impact of humanization on detection accuracy across feature clusters. The chart compares the detection accuracy of raw agent traces (light green) versus the minimum detectability achieved after applying humanization methods (dark green). The consistent drop in accuracy across most clusters demonstrates the effectiveness of the proposed methods in evading detection.}
    \label{fig:impact of humanization across feature clusters}
  \end{center}
  \vspace{-10pt}
\end{figure*}

Results in Table 3 reveal a clear performance hierarchy: \textit{History Matching} consistently outperforms \textit{B-spline Noise}. In Cluster 0, History Matching achieves a significantly lower detection score (0.7190 vs. 0.8507) and reduces XGBoost accuracy to 0.9450, compared to B-spline's 0.9817. \textbf{This confirms that adapting empirical human data provides a more robust defense than synthetic smoothing.}

\textit{Fake Action} decisively neutralizes interval-based rules, reducing accuracy to $\sim$0.52 (near random guessing). \textbf{However, this introduces a trade-off: while enhancing imitability, non-functional steps may disrupt task flow and degrade online utility.} Meanwhile, \textit{Long Press} effectively bypasses tap-duration heuristics, reducing detection accuracy in Cluster 1 from 0.9840 to 0.6133 by simply matching human duration distributions.


\textbf{Contrary to a simple modular ``plug-and-play'' assumption, our results reveal a complex interplay between strategies.} While \textit{Fake Action} successfully neutralizes interval-based rules, the second and fourth rows of the online results in Table~\ref{tab:main table} shows that its injection of fixed, repetitive motions can actually increase the overall detectability of the trajectory. This suggests that naive fake action injection lacks orthogonality; by introducing predictable mechanical artifacts, it facilitates detection in other feature dimensions. Consequently, achieving comprehensive imitability requires a more nuanced synchronization between temporal masking and trajectory generation to ensure that humanizing one dimension does not inadvertently compromise another.

\begin{figure*}[ht] 
    \centering
        \includegraphics[width=0.95\textwidth]{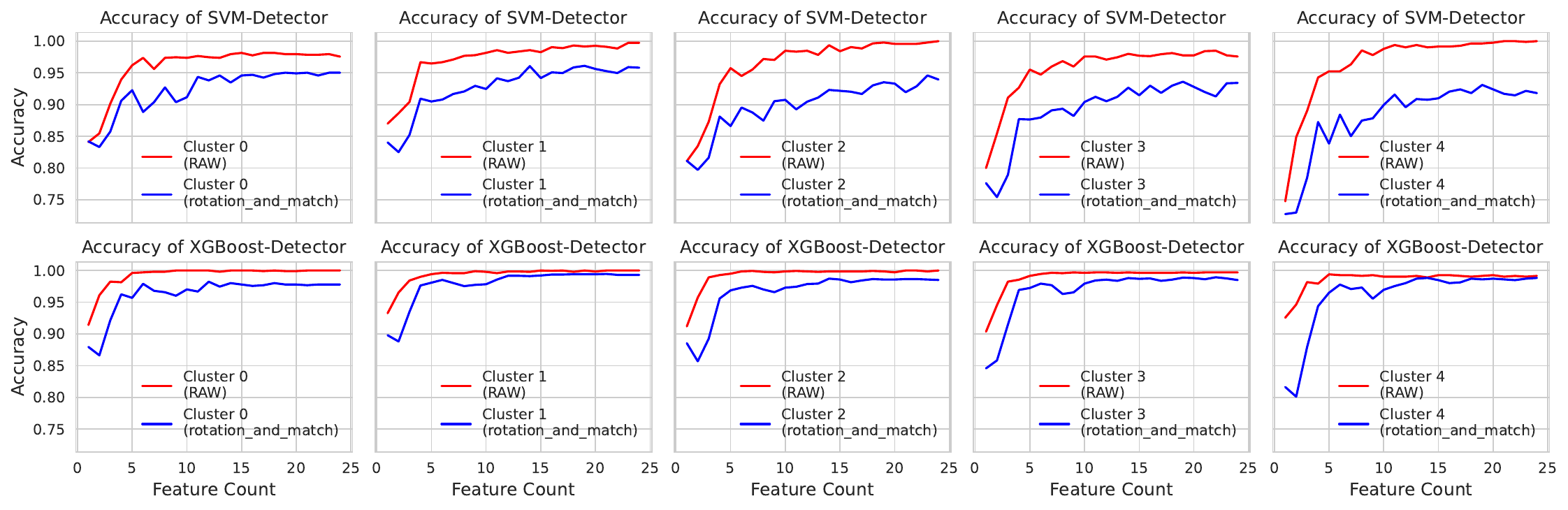}
        \label{fig:svm_acc}
    \caption{\textbf{Impact of Feature Selection on Detection Accuracy.} Comparison of (a) SVM and (b) XGBoost performance as the number of features increases.}
    \label{fig:accuracy_comparison}
\end{figure*}




\subsection{In-Depth Feature Analysis}
Table~\ref{tab:index0} evaluates the baseline (RAW) against three humanization strategies using an optimal ROC thresholding approach. Detection accuracy serves as the primary metric, where $1.0$ indicates perfect distinguishability and $0.5$ signifies successful humanization. To identify the limits of these strategies, we rank 24 behavioral features by their resistance to masking. As shown in Figure~\ref{fig:impact of humanization across feature clusters}, while certain attributes are easily humanized, others remain persistent bottlenecks tied to the fundamental architecture of GUI agents. Detailed results are available in Appendix~\ref{app:experiment results}.

\begin{table}[ht]
\centering
\vspace{-5pt}
\caption{Comparison of single-feature and model-based detection accuracy for Social Media. Results show the classification accuracy for the Raw baseline versus Online Rotation \& Match (On.RM), Offline Rotation \& Match (Off.RM), and B-Spline (BS) methods. Features are ranked by their discriminative power on the raw dataset.}
\label{tab:index0}
\setlength{\tabcolsep}{3pt} 

\begin{tabular}{lcccc}
\toprule
\textbf{Metric} & \textbf{RAW} & \textbf{On.RM} & \textbf{Off.RM} & \textbf{BS}\\
\midrule
maxDev & 0.9969 & 0.5515 & 0.6186 & 0.7556 \\
meanResultantLength & 0.9878 & 0.6818 & 0.6286 & 0.6979 \\
ratio\_end\_to\_len & 0.9878 & 0.6451 & 0.5798 & 0.5826 \\
duration & 0.8583 & 0.6907 & 0.5470 & 0.8507 \\
a20 & 0.8355 & 0.8286 & 0.7190 & 0.7686 \\
acc\_first5pct & 0.8244 & 0.5897 & 0.5532 & 0.8093 \\
a80 & 0.8154 & 0.6575 & 0.6205 & 0.6560 \\
dev80 & 0.7645 & 0.5310 & 0.6445 & 0.5476 \\
dev20 & 0.7634 & 0.5038 & 0.5560 & 0.5111 \\
dev50 & 0.7055 & 0.5510 & 0.6416 & 0.5251 \\
v80 & 0.6996 & 0.5188 & 0.6301 & 0.7026 \\
avgDirection & 0.6763 & 0.5515 & 0.5900 & 0.5872 \\
direction & 0.6734 & 0.5470 & 0.6734 & 0.5758 \\
startY & 0.6581 & 0.7472 & 0.6581 & 0.6581 \\
speed & 0.6531 & 0.5000 & 0.6100 & 0.6329 \\
startX & 0.6531 & 0.5310 & 0.6531 & 0.5730 \\
a50 & 0.6447 & 0.7789 & 0.6953 & 0.6501 \\
displacement & 0.6416 & 0.5154 & 0.6416 & 0.6387 \\
v50 & 0.6329 & 0.5407 & 0.6243 & 0.6387 \\
endX & 0.6243 & 0.5262 & 0.6243 & 0.5251 \\
endY & 0.6178 & 0.7279 & 0.6178 & 0.6232 \\
v20 & 0.6014 & 0.5479 & 0.5843 & 0.6014 \\
length & 0.5617 & 0.5271 & 0.5560 & 0.5588 \\
v\_last3\_median & 0.5419 & 0.6637 & 0.6313 & 0.5560 \\
svm\_accuracy & 0.9817 & 0.8750 & 0.9633 & 0.9633 \\
xgb\_accuracy & 1.0000 & 0.9773 & 0.9450 & 0.9817 \\
\bottomrule
\end{tabular}
\end{table}

\subsubsection{Feature Complexity vs. Detectability.} 
We randomly choose features as inputs for SVM and XGBoost to explore how the feature number affects detection accuracy. As shown in Figure~\ref{fig:accuracy_comparison}, detection accuracy scales logarithmically with feature count, plateauing after an initial steep rise. While RAW swipes reach near 100\% detectability with only 5--10 features, the \textit{rotation\_and\_match} strategy maintains a persistent detectability gap. Specifically, in SVM analysis, humanized accuracy remains between 0.85 and 0.95 even at 24 features. in XGBoost, despite faster convergence, humanized trajectories consistently underperform the RAW baseline. This confirms that our empirical humanization effectively masks mechanical patterns, forcing both linear and non-linear classifiers toward lower performance bounds across diverse contexts.

\subsubsection{The Easiest Features}
\textbf{Path shape metrics are the easiest to humanize.} As shown in Table~\ref{tab:index0}, content-agnostic metrics like \texttt{maxDev} and \texttt{ratio\_end\_to\_len} drop from near-perfect detectability ($\sim$0.99) to near-random levels (0.55--0.64) after humanization. By applying mathematical transformations like Rotation \& Match or B-splines, we can synthesize human-like motor imperfections without affecting task outcomes, making movement curvature the low-hanging fruit of humanization.

\subsubsection{The Hardest Features}
\label{sec:the hardest features}
\textbf{Precision and temporal rhythm remain resistant due to their direct conflict with Task Utility.} The first one is \textbf{Endpoint Precision.} While paths can be curved, scattering landing points is risky. In Table~\ref{tab:index0}, \texttt{endY} detectability actually doesn't decrease. Unlike humans who exhibit natural variance, agents target UI node centers. Intentionally adding coordinate noise without a granular understanding of clickable boundaries risks missing targets, forcing a trade-off between imitability and accuracy.

Another one is \textbf{Action Interval.} Agents suffer from LMM inference latency. While injecting fake actions aligns the temporal distribution with humans (see Fig.~\ref{fig:fake action interval} in Appendix~\ref{app:experiment results}), it severely compromises utility. As seen in the Trip Planning task (Table~\ref{tab:main table}), the accuracy plummets from 0.75 to 0.15. Blindly injected actions often trigger unintended state changes. Since the agent is unaware of these transitions, the original plan fails.

Hard-coded injections lack UI awareness, while agent-generated ones incur further latency. A promising future direction is to employ a lightweight \textit{Guard Agent} with sufficient visual understanding to identify safe zones for fake actions without the latency penalty of the main LMM.
\section{Related Works}
\label{sec:related_work_simpilified}

In this section, we provide related works. Due to space constraints, a comprehensive review of related works is provided in Appendix~\ref{app:related_works}.

\textbf{LMM-based GUI Agents}
Mobile automation has evolved from rigid scripts to autonomous agents powered by Large Multimodal Models (LMMs)~\cite{gpt4, gemini, llava}. Recent frameworks, such as AppAgent~\cite{appagent}, Mobile-Agent~\cite{mobile_agent,ye2508mobile}, CogAgent~\cite{cogagent}, and others~\cite{ma2024coco,li2025mobileuse,ma2024caution}, utilize visual perception to interact with interfaces, extending capabilities to web navigation~\cite{mind2web, yao2022webshop} and OS-level control~\cite{wu2025verios,cheng2025kairos}. However, these works prioritize \textbf{Task Success Rate (Utility)} via optimization techniques~\cite{gu2025mobile,lu2025arpo,xumobilerl}. Consequently, their motion control remains largely deterministic, creating a distinct behavioral gap compared to human users and leaving them vulnerable to detection.

\textbf{Adversarial Dynamics in Digital Ecosystems}
The conflict between agent efficiency and platform attention economies~\cite{lin2025superplatformsattackaiagents, allouah2025aiagentbuyingevaluation} has sparked adversarial dynamics. Existing research predominantly focuses on \textbf{Robustness} versus \textbf{Perturbation}~\cite{xi2025rise, wu2025dissectingadversarialrobustnessmultimodal, zhang2025agentsecuritybenchasb, xu2025advagentcontrollableblackboxredteaming}. Recent studies demonstrate attacks on visual grounding~\cite{gu2024agent,cui2023robustnesslargemultimodalmodels,dong2023robust}, ranging from environmental injections~\cite{liao2025eiaenvironmentalinjectionattack, chen2025evaluatingrobustnessmultimodalagents, chen2025obviousinvisiblethreatllmpowered, zhang2025attackingvisionlanguagecomputeragents} and visual adversaries~\cite{fang2024clipguidedgenerativenetworkstransferable, zhang2025qavaqueryagnosticvisualattack, de2024exploring} to backdoors~\cite{wang2024badagentinsertingactivatingbackdoor, yang2024watchagentsinvestigatingbackdoor, weng-etal-2025-foot}. Unlike these works which address functional availability, we focus on \textbf{survivability against behavioral detection}, framing the interaction as a ``Turing Test on Screen.''

\textbf{Bot Detection and Behavioral Biometrics}
Traditional bot detection~\cite{mahfouz2017survey,vastel2018fp,laperdrix2020browser} primarily identifies rigid scripts via deterministic patterns and fingerprints. In the mobile domain, detection leverages \textbf{Behavioral Biometrics}, utilizing touch dynamics (e.g., pressure, velocity) for user verification~\cite{zaidi2021touch, frank2012touchalytics, alrawili2024comprehensivesurveybiometricuser, Feng2012ContinuousMA, Kroeze2016UserAB, Shen2024IncreAuthIB}. While recent works extend these principles to mouse or game dynamics~\cite{khan2024mouse,pao2010game} and address robustness against replay or robotic attacks~\cite{ForgResTouchAuth, zaidi2021touch, serwadda2016toward, 9893211}, a critical gap remains regarding LMM-based agents. Unlike distinctively rigid bots or perfect replay attacks, LMM agents possess stochastic reasoning capabilities yet exhibiting mechanical execution, which current paradigms fail to address systematically.
\section{Discussion \& Future Work}
\label{sec:discussion}

The ``Turing Test on Screen'' serves not merely as a technical benchmark, but as the prelude to a long-term evolutionary arms race between digital platforms and autonomous agents. In this concluding section, we discuss the anticipated trajectories of this conflict from both defensive and offensive perspectives.

\subsection{The Future of Agent Humanization}

To survive within this escalated detection landscape, agent humanization techniques must evolve beyond simple trajectory smoothing. We identify three key directions for future research:

\subsubsection{From Post-Processing to End-to-End Humanization}
The Wrapper approach adopted in this study faces an inherent trade-off between \textbf{Offline Quality} and \textbf{Online Latency}. Retrieving and adapting high-fidelity human trajectories introduces computational overhead. In real-time environments, this latency may cause the agent to miss transient UI events such as a closing popup window, thereby negatively impacting the Task Success Rate.

We posit that humanization should be intrinsic to the model architecture itself. Rather than relying on latency-inducing post-processing, future Large Multimodal Models should be trained or fine-tuned to generate humanized trajectories directly via an end-to-end framework.

\subsubsection{Personalized Humanization}
Detection algorithms may eventually advance to \textbf{Personalized Detection}, verifying not merely whether a user is human, but whether the behavior matches the specific user's historical profile. Consequently, agents must advance towards \textbf{Personalized Humanization}, where the system learns to mimic the unique motor patterns and behavioral habits of its specific user rather than a generic population average.

\subsubsection{Generalized Cross-Modal Humanization} Finally, human interaction is fundamentally multimodal. While our current benchmark prioritizes Touch and Swipe events, future iterations of the AHB should extend their scope to encompass additional modalities. Specifically, this includes \textbf{Typing Dynamics}, which entails simulating keystroke rhythms defined by realistic error rates and inter-key latency variations. Furthermore, it is essential to model \textbf{Scrolling and Reading Behaviors}, where scroll velocity modulates in response to content density rather than maintaining an artificial constant speed.




\subsubsection{AHB as an Evolutionary Compass}
Ultimately, the Agent Humanization Benchmark (AHB) transcends its role as a mere evaluation metric to become a cornerstone of a new \textbf{survival-centric design philosophy} for GUI agents. By quantifying the trade-off between \textbf{Imitability} and \textbf{Utility}, AHB serves as a fitness function that drives a paradigm shift: from the singular pursuit of efficiency to a dual-objective optimization of architectural resilience and behavioral camouflage. 

In the evolving arms race between platforms and user agency, AHB guides the development of \textbf{indistinguishable digital citizens}next-generation agents that possess both the functional power to assist users and the behavioral nuance required to coexist harmoniously in adversarial digital ecosystems.

\subsection{The Future of Agent Detection}

Current detection methodologies predominantly operate at the \textbf{Execution Layer}, scrutinizing the kinematic fidelity of individual actions. However, as humanization strategies approximate motor perfection, the biometric surface between humans and agents will blur. We posit that the adversarial frontier may shift to the \textbf{Intent Layer}.

Consequently, the ultimate form of the Turing Test on Screen will evolve from distinguishing whose hand is moving to determining \textbf{whose brain is thinking.} Future detectors are expected to model behavioral sequences over longer horizons, seeking signs of human curiosity, distraction, and indecision that algorithmic efficiency inherently strives to eliminate.

More discussion on 1. the robustness of detection baselines  2. delineating the scope from motion dynamics to physical sensors 3. the imitability-utility pareto frontier 4. broder impact and ethical consideration  are provided in Appendix~\ref{app:discussion}.

\section{Conclusion}
This paper introduces the ``Turing Test on Screen,'' a novel paradigm evaluating the anti-detection capabilities of autonomous GUI agents through behavioral humanization. By formalizing agent-platform interactions as a MinMax optimization problem, we established the Agent Humanization Benchmark (AHB) and a high-fidelity dataset to quantify the trade-off between imitability and task utility. Our findings demonstrate that while vanilla LMM-based agents are highly detectable due to kinematic anomalies, our proposed humanization strategies significantly enhance behavioral authenticity both theoretically and empirically while maintaining performance.

As the adversarial landscape evolves, we anticipate a paradigm shift in detection from execution-level kinematics to intent-level patterns. The AHB serves as a compass for this transition, guiding the development of agents that move beyond mere functional efficiency toward seamless coexistence within human-centric digital infrastructures. Ultimately, this work lays the foundation for autonomous agents to sustainably safeguard user agency in increasingly adversarial online environments.

\bibliographystyle{unsrt}  
\bibliography{references}  






\appendix

\section{The Conflict Between GUI Agents and App Platforms}
\label{app:doubao_case}

\subsection{Background}
To understand the gravity of this incident, it is essential to contextualize the hardware involved. The conflict centers on devices like the \textbf{Nubia M153} smartphone, which features a deep integration of ByteDance's \textit{Doubao Mobile Assistant}. 

Unlike traditional voice assistants (e.g., Siri or Google Assistant) that largely rely on official APIs, Doubao functions as a system-level agent. It utilizes the \texttt{INJECT\_EVENTS} permission and Vision-Language Models (VLM) to ``see'' the screen and simulate physical taps. This allows it to execute cross-app workflows without manual input, promising a ``zero-touch'' experience for commands such as:
\begin{quote}
    \textit{``Open WeChat and send a message to my wife saying I'll be home in 20 minutes.''}
\end{quote}
For the end-user, this flattens the friction of navigating multiple apps. For app developers, however, this unauthorized ``driving'' of their interface represents a significant security gray area.

\begin{figure*}[htbp]
  \begin{center}
    \centerline{\includegraphics[width=0.95\textwidth]{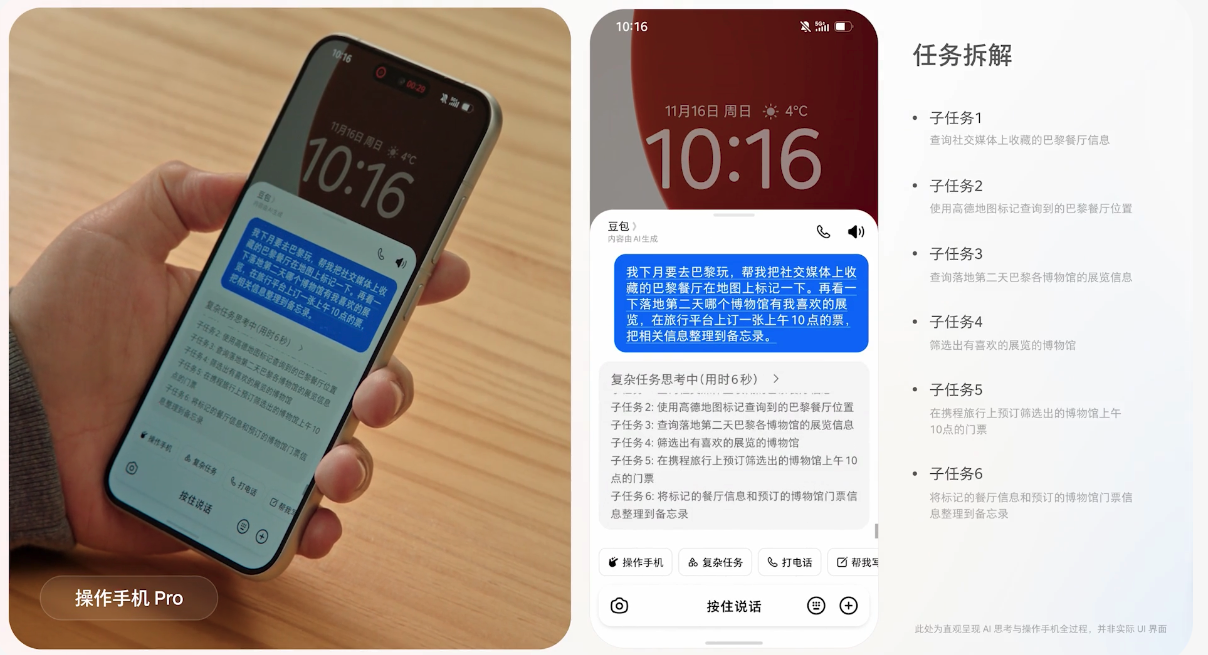}}
    \caption{Doubao Mobile Assistant Working Scene on the Offical Website.}
    \label{fig:fake action interval}
  \end{center}
\end{figure*}

\subsection{The Incident}
In late 2025, the theoretical tension between GUI agents and app platforms escalated into operational failures.\footnote{https://opentools.ai/news/wechat-and-chinese-banking-apps-check-bytedances-doubao-mobile-assistant-amid-privacy-and-security-concerns}

\subsubsection{The Automation Scenario}
The Doubao agent allows users to automate complex interaction chains via voice. For example, instead of manually opening WeChat, finding a contact, and typing, the user issues a voice command. The AI then visibly takes over the screen, opening windows and simulating clicks to execute the task. 

While ByteDance explicitly states that the assistant \textbf{does not perform sensitive operations like payments or identity verification}, the agent still requires deep access to the app's GUI to function for standard tasks like messaging or searching.

\subsubsection{The Security Backlash}
Upon release, users of the Nubia M153 immediately encountered service denials, including forced logouts from WeChat and security warnings from banking applications like the \textit{Agricultural Bank of China}. The conflict was driven by the assistant's reliance on the \texttt{INJECT\_EVENTS} permission, a high-privilege capability that allows software to programmatically generate touch inputs and keystrokes. Consequently, the apps' automated risk control measures interpreted the AI's external control as unauthorized manipulation, triggering defensive protocols designed to prevent account hijacking and fraud.

\subsection{Technical Mechanism of Conflict}
The conflict stems from a fundamental incompatibility between \textbf{agentic automation} and \textbf{platform protocols}. Since app platforms like WeChat do not provide open APIs for third-party control, agents must resort to screen-driven techniques, treating the application simply as a graphical user interface to be navigated visually. However, this approach directly clashes with the platforms' defensive architectures, which employ sophisticated anti-bot algorithms to prevent bulk-spamming and fraud. From the platform's perspective, the agent's interaction patterns were indistinguishable from malicious scripts; consequently, the system prioritized security, defaulting to blocking any non-human input to prevent potential data breaches.

\subsection{Arguments and Implications}
The standoff highlights the divergent incentives of the two parties:

\begin{enumerate}[leftmargin=15pt]
    \item \textbf{The OS/Agent Provider (ByteDance/Nubia):} They argue for \textit{User Agency} and \textit{Innovation}. They contend that since the user explicitly authorized the assistant, the AI acts as a legitimate digital proxy for human intent. ByteDance further emphasized that their tool adheres to privacy standards and deliberately avoids sensitive operations like financial transactions.
    
    \item \textbf{The Super-Platform (Tencent/Banks):} They cite \textit{Security and Ecosystem Integrity}. Reports indicate that WeChat's restrictions were not specifically targeted at Doubao but were unintentional triggers of existing risk control measures. They implies that allowing external programs to drive the apps bypasses critical security checks, creating a vulnerability that could be exploited by malicious actors if such automation becomes normalized.
\end{enumerate}

\subsection{Outcome}
The \textbf{Doubao Incident} serves as a critical case study for the GUI Agent industry. It demonstrates that permissionless UI-based automation remains a fragile operational mode. Without formal API agreements~\cite{yang2025surveyaiagentprotocols} or standardized security protocols, AI agents attempting to navigate walled gardens will inevitably collide with the defensive countermeasures of established software ecosystems.

\section{Feature Extraction and Statistical Analysis}
\label{sec:feature_extraction}

\begin{table*}[h]
    \centering
    \caption{Task Clusters and Applications}
    \renewcommand{\arraystretch}{0.9}
    \begin{tabular}{lll}
        \toprule
        \textbf{Cluster ID} & \textbf{Category} & \textbf{Applications} \\
        \midrule
        0 & Social Media & Toutiao, Weibo, Xiaohongshu, Zhihu \\
         1 & Shopping & JD, Taobao, Cainiao, Meituan, Eleme \\
         2 & Video Streaming & iQIYI, Bilibili, QQ Music \\
         3 & Trip Planning & Ctrip, Amap (Gaode), Umetrip, Qunar \\
         4 & Office \& Learning & Tencent Docs, Tencent Meeting, Youdao Dictionary, Haodafu \\
        \bottomrule
    \end{tabular}
    \label{tab:clusters and applications}
\end{table*}

In this section, we delineate the methodology for extracting discriminative behavioral features from the raw event streams $\mathcal{E}$ defined in Section~\ref{sec:problem_formulation}. Following our adversarial framework, the fundamental unit of behavioral analysis is the \textbf{action} $a_t$. Each action, whether performed by a human or a GUI agent, triggers a corresponding event set $E_t$.

Mathematically, each action is characterized by its triggered signature: $E_t = (M_t, S_t)$, where $M_t$ and $S_t$ denote the synchronized \textit{MotionEvents} and \textit{SensorEvents}, respectively. The \textit{MotionEvent} component $M_t$ consists of a series of \textit{FingerEvents} $\{f_n\}_{n=1}^N$, where each $f_n = (x_n, y_n, t_n)$ encodes the spatial coordinates and timestamp at sample $n$. Concurrently, the \textit{SensorEvent} component $S_t$ captures high-dimensional multimodal data generated during the execution of $a_t$:
\begin{equation}
    S_t = \{Acc, Gyro, RotVec, Grav, LinAcc, Mag, Light, Prox\}
\end{equation}
where these elements correspond to the Accelerometer, Gyroscope, and other hardware sensors. Intervals between consecutive actions represent periods of inactivity during state transitions $s_t \to s_{t+1}$.

\subsection{Dataset Collection}

Based on the definition of $a_t$, we construct a large-scale, multi-modal dataset to evaluate the discriminative power of the Detector $D_{\Theta}$. Table~\ref{tab:dataset_composition} and Table~\ref{tab:dataset_composition_humanized_agents} provides a detailed breakdown of the dataset, which spans two primary operator distributions:
\begin{itemize}[leftmargin=10pt]
    \item \textbf{Human Operators ($\mathcal{H}$):} Data collected from four distinct sub-populations (Young Man, Young Woman, Middle-aged, and Elderly) to account for physiological variances in human actions.
    \item \textbf{GUI Agents ($G_{\Phi}$):} Action sequences recorded from state-of-the-art models including UI-TARS, MobileAgent-E (GPT-4o/Claude-3.5-Sonnet), AgentCPM and AutoGLM.
\end{itemize}

All experiments are conducted on the same device—a \textbf{Xiaomi Mi Max 2} running MIUI 11.0.2.0—to ensure consistency and comparability. Data are collected online, verbatim from the phone, while humanization methods are applied in real time during agent actions rather than post hoc, ensuring accurate assessment of their effects.

For humanized agents without fake actions, tap durations are elongated, and swipes are rendered realistic via data-driven trajectory matching. As shown in Table~\ref{tab:main table}, some agents undergo only swipe humanization, leaving tap durations unmodified. For agents with fake actions, the applicable humanization techniques (tap elongation and/or swipe humanization) are augmented with small circular gestures (radius: 50 px) emitted from the last tap location according to a Poisson process with rate $\lambda = 0.9$ Hz.

Sensor data recorded alongside interactions include: \textbf{Accelerometer} (proper acceleration, $m/s^2$), \textbf{Gyroscope} (angular velocity, rad/s), \textbf{Magnetic Field} (geomagnetic field, $\mu$T), \textbf{Gravity} (estimated gravitational acceleration, derived from accelerometer), \textbf{Linear Acceleration} (acceleration excluding gravity, obtained by subtracting Gravity from Accelerometer), \textbf{Rotation Vector} (fused orientation from Accelerometer, Gyroscope, and Magnetic Field), \textbf{Light} (ambient illuminance in lux), and \textbf{Proximity} (nearby object detection). The latter two are hardware sensors; the rest include both physical and virtual (software-fused) types. Visualization of sensor changes across time axis is shown in Figure~\ref{fig:Sensor Gestures Visulization}.

For more details, see the official documentation:  
\href{https://developer.android.com/develop/sensors-and-location/sensors/sensors_overview}{Android Sensor Overview} and  
\href{https://source.android.com/docs/core/interaction/sensors/sensor-types}{Android Sensor Types}.


\begin{figure}[ht]
  \begin{center}
    \centerline{\includegraphics[width=0.9\textwidth]{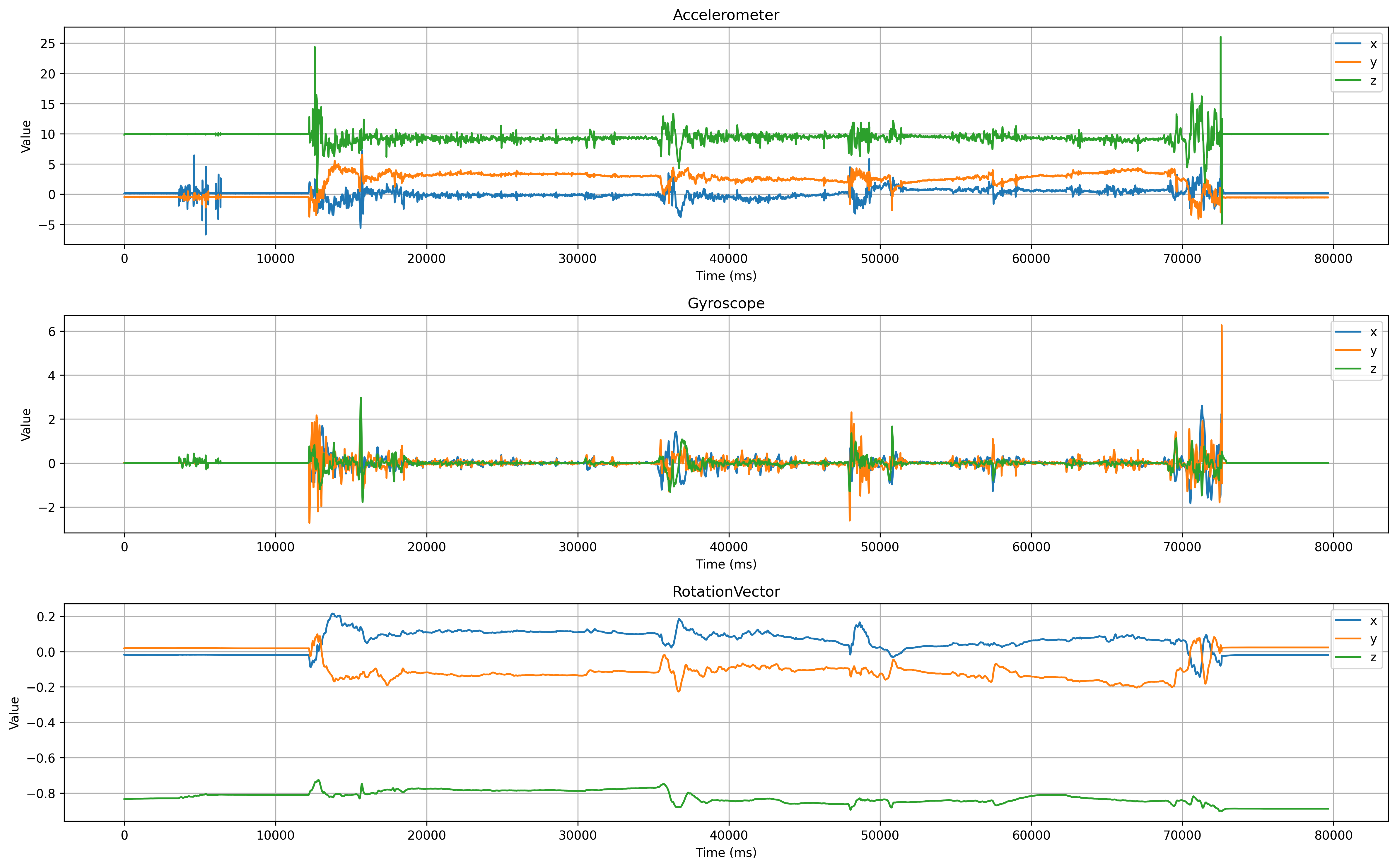}}
    \caption{Sensor Events Visulization as Time Changes.}
    \vspace{-10pt}
    \label{fig:Sensor Gestures Visulization}
  \end{center}
\end{figure}

However, it is worth noting that while our dataset encompasses both event types, achieving high-fidelity humanization of \textit{SensorEvents} poses significant technical challenges. In realistic deployment scenarios, such as when a mobile device is placed stationary on a flat surface, the intricate fluctuations of sensors like the Gyroscope and Magnetic Field are inherently difficult for GUI agents to simulate authentically. The only viable path for an agent to simulate such signals would require system-level API interventions to inject synthetic sensor values. Consequently, this study intentionally constrains its primary focus to the humanization of \textit{MotionEvents}, treating the investigation of sensor-level adversarial simulation as a secondary objective to be addressed in future work.
\subsection{The Definition of Tap and Swipe Actions}

We categorize each action $a_t$ as either a \textit{tap} or a \textit{swipe} based on the number of its \textit{FingerEvents} within $M_t$. Specifically, an action is defined as a tap if $|f| < 5$ and a swipe if $|f| \ge 5$. We visualize every action from the perspective of length and duration. As is shown in Figure~\ref{fig:tap/swipe definition}, the duration of an action is nearly proportional to its length, since—given the smartphone’s constant MotionEvent sampling rate and the fact that finger movement during swipes continuously generates new events—the effective MotionEvent generation rate remains roughly constant. Moreover, we can see many actions with length $<5$ clustered in the bottom-left corner of the figure; these are classified as taps, consistent with the observation that taps involve little to no finger movement between down and up events.

\begin{figure}[ht]
    \begin{center}
        \centerline{\includegraphics[width=0.9\textwidth]{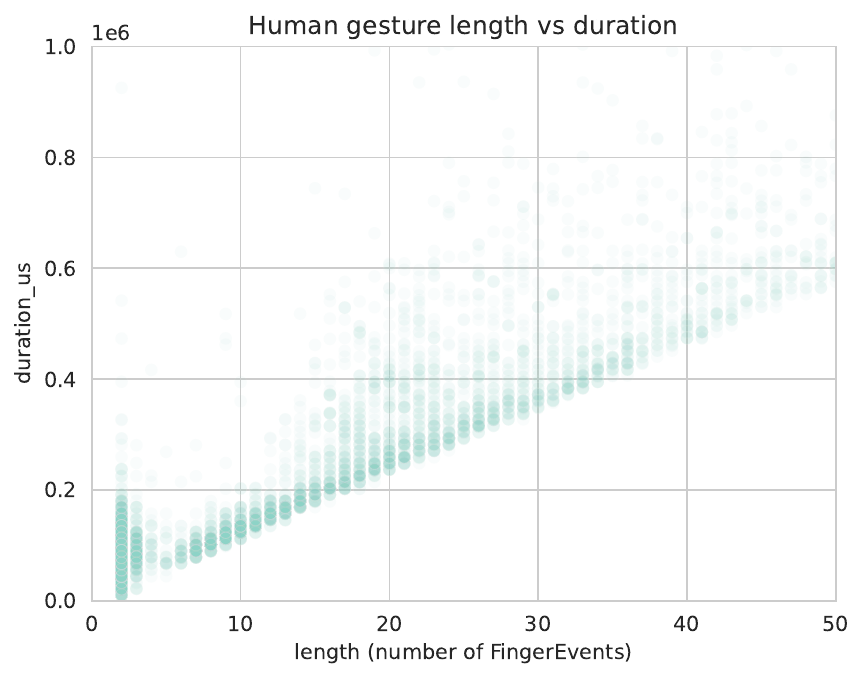}}
        \caption{The Lengths and Durations of Each Action. Actions with $|f| < 5$ are considered taps.}
        \label{fig:tap/swipe definition}
    \end{center}
\end{figure}

\begin{table*}[t]
\centering
\caption{Dataset composition of humans and raw agents across motion and sensor events.}
\label{tab:dataset_composition}
\renewcommand{\arraystretch}{1.2}
\setlength{\tabcolsep}{3pt} 
\resizebox{0.95\textwidth}{!}{%
\begin{tabular}{c|c|cc|cccccccc|c}
\toprule
\multicolumn{2}{c|}{Category} & \multicolumn{2}{c|}{Motion Event} & \multicolumn{8}{c|}{Sensor Event} & \multirow{2}{*}{Overall} \\ 
\cmidrule(lr){1-12}
Type & Sub-population & Tap & Swipe & Accel. & Gyro. & RotVec. & Grav. & LinAcc. & Mag. & Light & Prox. & \\
\midrule
\multirow{4}{*}{Human} 
& Young Man     & 5050 & 2186 & 4516935 & 4516491 & 4482116 & 4516352 & 4516349 & 1101702 & 5300 & 566 & \multirow{4}{*}{37,768,698} \\ 
& Young Woman   & 1706 & 786  & 1300213 & 1300038 & 1279412 & 1300018 & 1300012 & 314457  & 2544 & 182 & \\
& Middle-aged   & 782  & 208  & 743043  & 742953  & 742922  & 742953  & 742952  & 182605  & 1314 & 86  & \\
& Elderly       & 644  & 161  & 651490  & 651402  & 641091  & 651399  & 651396  & 157573  & 1225 & 84  & \\
\midrule
\multirow{5}{*}{Agent} 
& UI-TARS                          & 772  & 380  & 2,991,281 & 2,991,215 & 2,971,001 & 2,991,077 & 2,991,074 & 730,144  & 579  & 212 & \multirow{5}{*}{243,090,929} \\ 
& Mobile-Agent-E (GPT-4o)          & 832  & 148  & 15,392,992& 15,392,962& 15,392,742& 15,392,802& 15,392,802& 3,782,534 & 2,999 & 222 & \\
& Mobile-Agent-E (Claude-Sonnet)  & 849  & 141  & 16,165,856& 16,165,847& 15,655,597& 16,165,637& 16,165,636& 3,846,893 & 1,185 & 212 & \\
& AgentCPM                        & 2,400 & 166  & 3,343,441 & 3,343,425 & 3,343,258 & 3,343,303 & 3,343,305 & 821,607  & 472  & 175 & \\
& AutoGLM                          & 1,597 & 339  & 8,577,328 & 8,577,289 & 8,550,015 & 8,577,154 & 8,577,152 & 2,101,041 & 619  & 220 & \\
\bottomrule
\end{tabular}
}
\end{table*}

\begin{table*}[t]
\centering
\caption{Dataset composition of humanized agents with or without fake actions across motion and sensor events.}
\label{tab:dataset_composition_humanized_agents}
\renewcommand{\arraystretch}{1.2}
\setlength{\tabcolsep}{3pt} 
\resizebox{0.95\textwidth}{!}{%
\begin{tabular}{c|c|cc|cccccccc|c}
\toprule
\multicolumn{2}{c|}{Category} & \multicolumn{2}{c|}{Motion Event} & \multicolumn{8}{c|}{Sensor Event} & \multirow{2}{*}{Overall} \\ 
\cmidrule(lr){1-12}
Type & Sub-population & Tap & Swipe & Accel. & Gyro. & RotVec. & Grav. & LinAcc. & Mag. & Light & Prox. & \\
\midrule
\multirow{5}{*}{\shortstack{Humanized \\Agent\\(w/o fake\\ action)}} 
& UI-TARS                          & 1677 & 683  & 6357826 & 6357750 & 6318637 & 6357479 & 6357484 & 1552891 & 761  & 388 & \multirow{5}{*}{280,320,486} \\ 
& Mobile-Agent-E (GPT-4o)          & 1382 & 269  & 22018165& 22018033& 21964700& 22017779& 22017784& 5397465 & 1354 & 401 & \\
& Mobile-Agent-E (Claude-Sonnet)  & 675  & 46   & 12942360& 12942301& 12141290& 12942167& 12942158& 2983499 & 2394 & 215 & \\
& AgentCPM                        & 1954 & 286  & 3583409 & 3583366 & 3583216 & 3583246 & 3583252 & 880616  & 420  & 174 & \\
& AutoGLM                          & 1797 & 348  & 8752376 & 8752336 & 8724506 & 8752198 & 8752196 & 2143920 & 632  & 225 & \\
\midrule
\multirow{3}{*}{\shortstack{Humanized \\Agent\\(with fake\\ action)}}
& UI-TARS                          & 735  & 6001 & 3213798 & 3213773 & 3200123 & 3213647 & 3213644 & 786520  & 2833 & 185 & \multirow{3}{*}{154,704,094} \\ 
& Mobile-Agent-E (GPT-4o)          & 601  & 33187& 14346123& 14346063& 14345898& 14345959& 14345962& 3525565 & 425  & 178 & \\
& AutoGLM                          & 2201 & 21521& 11921876& 11921824& 11921638& 11921682& 11921681& 2929341 & 888  & 222 & \\
\bottomrule
\end{tabular}
}
\end{table*}

\begin{table*}[ht]
    \centering
    \small 
    \caption{Touch Dynamics Feature Descriptions with Information Gain (IG)}
    \label{tab:touch_features}
    \renewcommand{\arraystretch}{1.2} 
    \begin{tabular}{l l c p{6cm}}
        \toprule
        \textbf{Variable Name} & \textbf{Full Description} & \textbf{Information Gain} & \textbf{Explanation} \\
        \midrule
        
        \multicolumn{4}{l}{\textit{\textbf{Velocity (Kinematics)}}} \\
        v20 & 20\%-perc. pairwise velocity & 0.4487 & Speed at lower 20th percentile (slow phase). \\
        v50 & 50\%-perc. pairwise velocity & 0.4491 & Median velocity (typical speed). \\
        v80 & 80\%-perc. pairwise velocity & 0.4750 & Speed at 80th percentile (peak speed). \\
        speed & average velocity & 0.4238 & Total distance divided by total duration. \\
        v\_last3\_median & median velocity at last 3 pts & 0.4796 & Deceleration behavior near target. \\
        \midrule
        
        \multicolumn{4}{l}{\textit{\textbf{Acceleration (Kinematics)}}} \\
        a20 & 20\%-perc. pairwise acc & 0.4555 & Acceleration at lower 20th percentile. \\
        a50 & 50\%-perc. pairwise acc & 0.3583 & Median acceleration (smoothness). \\
        a80 & 80\%-perc. pairwise acc & 0.3616 & Acceleration at 80th percentile (jerks). \\
        acc\_first5pct\_median & median acc. at first 5 pts & 0.4197 & Initial impulse force at touch-down. \\
        \midrule
        
        \multicolumn{4}{l}{\textit{\textbf{Deviation \& Linearity}}} \\
        dev20 & 20\%-perc. dev. from line & 0.5406 & Lower bound deviation from straight line. \\
        dev50 & 50\%-perc. dev. from line & 0.5911 & Median deviation (general straightness). \\
        dev80 & 80\%-perc. dev. from line & 0.5084 & Major deviation (curves). \\
        maxDev & largest dev. from line & 0.6726 & Furthest point from ideal line (max arc). \\
        \midrule
        
        \multicolumn{4}{l}{\textit{\textbf{Geometry \& Spatial}}} \\
        length & length of trajectory & 0.4956 & Actual path length traveled. \\
        displacement & direct end-to-end distance & 0.5178 & Linear distance between start and end. \\
        ratio\_end\_to\_length & ratio end-to-end dist / length & 0.6581 & Path efficiency (1 = perfect straight line). \\
        meanResultantLength & mean resultant length & 0.6581 & Directional consistency statistic. \\
        direction & direction of end-to-end line & 0.6052 & Global angle of the action. \\
        avgDirection & average direction & 0.6107 & Average of local angles between points. \\
        \midrule
        
        \multicolumn{4}{l}{\textit{\textbf{Coordinates \& Time}}} \\
        startX / startY & start x / start y & 0.5167 / 0.4036 & Touch-down coordinates. \\
        endX / endY & stop x / stop y & 0.5160 / 0.3975 & Lift-off coordinates. \\
        duration & action duration & 0.5831 & Total time in milliseconds. \\
        \bottomrule
    \end{tabular}
\end{table*}

\begin{figure*}[ht]
  \begin{center}
    \centerline{\includegraphics[width=0.95\textwidth]{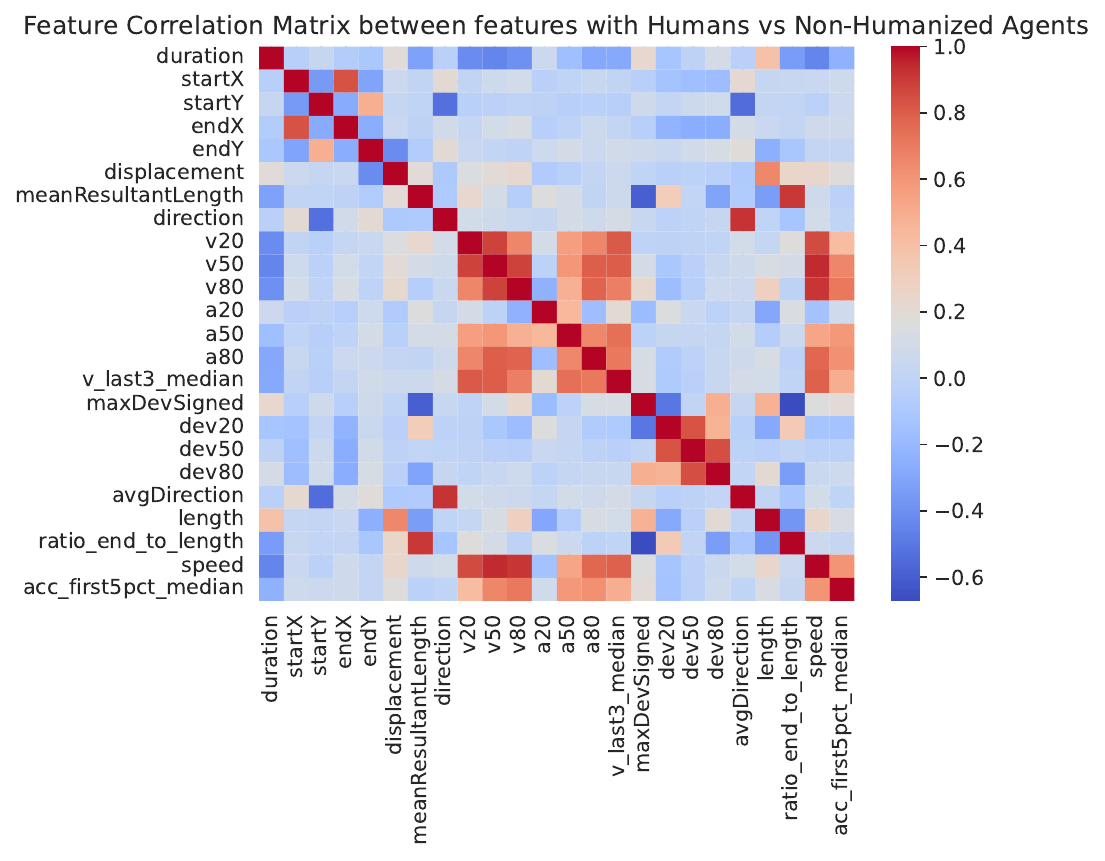}}
    \caption{The correlation of these 24 features. Red color means stronger correlation.}
    \label{fig:touch_features}
  \end{center}
\end{figure*}

\subsection{24 statistical features}

From these raw sequences, we derive 24 statistical features detailed in Table~\ref{tab:touch_features}, selected to capture specific biomechanical and behavioral characteristics. The rationale for including the absolute spatial coordinates of the trajectory's start and end points stems from the observation that users exhibit unique spatial preferences, frequently interacting with specific areas on the screen regardless of the underlying interface layout.

We utilize kinematic features to capture dynamic motion control habits. Specifically, by calculating\textbf{ the average velocity of the last five points of the trajectory}, we can distinguish between two different rolling behaviors: static release, where the user stops moving before lifting their finger and ballistic release, where the user maintains a certain lateral velocity when lifting their finger, resulting in inertial rolling. This distinction is typically associated with a unique throwing velocity inherent to a particular user.

To comprehensively quantify the geometric linearity and curvature of the action, we utilize the \textbf{mean resultant length} to measure how directed the action is. Specifically, all $N$ consecutive coordinate pairs $(x_n, y_n), (x_{n+1}, y_{n+1})$ along the path define an ensemble of $N - 1$ unit direction vectors $z_n = \exp(i\phi_n)$ with angles $\phi_n$. The mean resultant length $R$ of this ensemble is formally characterized by:
\begin{equation}
    R = (N - 1)^{-1}\left|\sum_{n=1}^{N-1} z_n\right|
\end{equation}
This metric scales between 1 and 0, indicating a perfectly straight line and uniformly random angles, respectively, and provide a robust measure of angular dispersion. Associated with this ensemble is the \textbf{mean direction}, defined as $\arg((N - 1)^{-1}\sum_{n=1}^{N-1} z_n)$.

In addition, we calculated\textbf{ the path efficiency ratio}, defined as the Euclidean distance between the endpoints divided by the total trajectory length, and the \textbf{maximum signed perpendicular deviation of the trajectory} relative to the ideal line connecting the endpoints. The signed deviation is particularly useful for identifying the convexity of action arcs, which can serve as a potential indicator of the user's dominant hand.

Temporal dynamics provide further behavioral resolution. Features such as\textbf{ action duration and inter-action latency} act as proxies for cognitive processing and reading speed, distinguishing between users and agents who employ slow, continuous scrolling and those who execute rapid, discrete page shifts.

\subsection{Information Gain}
\label{sec:feature_informativeness}

To assess the utility of the extracted feature set in distinguishing between authentic human inputs and agent operations, we analyze the information gain provided by each individual attribute. Consistent with findings in behavioral biometrics, not all features contribute equally; therefore, we quantify a measure of \textbf{informativeness} ($I_F$) for each feature $F$ to reveal the hierarchy of feature relevance.

We define this measure as the relative mutual information between the feature variable $F$ and the source identity variable $U$ (representing the class: Human or Agent). Formally, $I_F$ is calculated as the ratio of the mutual information $I(F;U)$ to the entropy of the source identity $H(U)$:
\begin{equation}
    I_F := \frac{I(F;U)}{H(U)} = \frac{H(U) - H(U|F)}{H(U)} = 1 - \frac{H(U|F)}{H(U)}
\end{equation}
Here, $H(U)$ and $H(U|F)$ denote the marginal entropy of the source identity and the conditional entropy of the source identity given feature $F$, respectively. This normalized metric yields a value between 0 and 1, where 0 indicates that the feature carries no discriminatory information, and 1 implies that the feature strictly determines the source identity.

Based on the quantitative analysis in Table \ref{tab:touch_features}, geometric features (e.g., \textbf{maxDev}, $IG \approx 0.66$) demonstrate the highest discriminative power, effectively capturing the contrast between natural human motor noise and agent linearity. In contrast, acceleration statistics and absolute coordinates exhibit low informativeness ($IG < 0.40$) due to their susceptibility to sensor noise and UI layout dependency. 

Theoretically, features with higher mutual information are more beneficial for distinguishing humans from agents. However, robust detection is not achieved by individual metrics alone; rather, one can gain more information by combining features that complement each other to capture the full behavioral patterns.

To visualize potential redundancy within the feature space, Figure~\ref{fig:touch_features} presents a heatmap of pairwise correlation coefficients. The color intensity reflects the magnitude of the correlation: darker red denotes strong positive relationships, darker blue indicates strong negative associations, and white signifies statistical independence. While this matrix reveals distinct clusters of highly correlated attributes, we do not exclude features based solely on these metrics. It is well-established that correlated variables can still contribute unique, complementary information when combined in high-dimensional classification models. Consequently, we retain the complete feature set for our SVM and XGBoost detectors. The primary utility of this visualization is to empirically verify specific behavioral hypotheses, such as the consistent positive synchronization observed across the 20\%, 50\%, and 80\% velocity percentiles.

\section{Theoretical Results}
\label{app:theory}

In this section, we provide a bunch of theoretical theorem to prove the effectiveness of heuristic Noise Injection and data-driven history matching

Let $(\mathcal{X}, \mathcal{F})$ be the measurable space of behavioral trajectories. Let $P \in \mathcal{P}(\mathcal{X})$ denote the true distribution of Human behaviors, and $G_\Phi \in \mathcal{P}(\mathcal{X})$ denote the distribution of the Raw Agent parameterized by $\Phi$. The Humanization Wrapper is modeled as a Markov transition kernel $K$, yielding the humanized agent distribution $G'_\Phi = G_\Phi K$.

We posit two fundamental assumptions. First, we assume the human distribution $P$ is absolutely continuous, whereas the Raw Agent distribution $G_\Phi$ lies on a lower-dimensional manifold, implying they are mutually singular ($P \perp G_\Phi$). Second, we restrict the detector $D_\Theta$ to the class $\mathcal{D}_L$ of $L$-Lipschitz continuous functions. This is a standard assumption for neural networks with bounded weights.

The detector aims to maximize the standard cross-entropy objective:
\begin{equation}
\label{eq:V}
\mathcal{L}_D(D_\Theta; P, G'_\Phi) = \mathbb{E}_{x \sim P}[\log D_\Theta(x)] + \mathbb{E}_{x \sim G'_\Phi}[\log(1 - D_\Theta(x))].
\end{equation}

The following theorem restates the classical result regarding the detector's optimal performance.

\begin{theorem}
\label{thm:opt_detector}
For any fixed agent policy $G'_\Phi$, the maximum discrimination capability of the detector is bounded by the Jensen--Shannon divergence:
\begin{equation}
    \sup_{D_\Theta} \mathcal{L}_D(D_\Theta; P, G'_\Phi) = -\log 4 + 2 \cdot \JS(P \parallel G'_\Phi).
\end{equation}
\end{theorem}
\begin{proof}
Following the derivation in \cite{goodfellow2014generative}, consider the integrand $f(y) = p(x) \log y + q(x) \log(1-y)$ for a fixed $x$. Setting $\frac{\partial f}{\partial y} = 0$ yields the optimal discriminator $D^*(x) = \frac{p(x)}{p(x)+q(x)}$. Substituting $D^*$ back into Eq.~\eqref{eq:V}:
\begin{align*}
    \sup_{D} \mathcal{L}_D &= \int p(x) \log \frac{p(x)}{p(x)+q(x)} dx + \int q(x) \log \frac{q(x)}{p(x)+q(x)} dx \\
    &= \int p(x) \log \frac{p(x)}{\frac{p(x)+q(x)}{2}} dx + \int q(x) \log \frac{q(x)}{\frac{p(x)+q(x)}{2}} dx - 2\log 2 \\
    &= \KL(P \parallel M) + \KL(G'_\Phi \parallel M) - \log 4 = -\log 4 + 2 \cdot \JS(P \parallel G'_\Phi),
\end{align*}
where $M = (P+G'_\Phi)/2$.
\end{proof}

We now analyze our first strategy, which injects variance via randomized smoothing.

\begin{theorem}
\label{thm:smoothing}
Let $K_\sigma$ be a strictly positive smoothing kernel and $G'_\Phi = G_\Phi * K_\sigma$. Then the smoothed agent strictly reduces the maximum theoretical detectability:
\begin{equation}
    \JS(P \parallel G'_\Phi) < \JS(P \parallel G_\Phi) = \log 2.
\end{equation}
\end{theorem}
\begin{proof}
For the Raw Agent, $P \perp G_\Phi$ implies disjoint supports. Thus, the JS divergence is maximized:
\[
\JS(P \parallel G_\Phi) = \frac{1}{2} \int_{\Supp(P)} p \log \frac{2p}{p} + \frac{1}{2} \int_{\Supp(G_\Phi)} g \log \frac{2g}{g} = \log 2.
\]
For the smoothed agent, $G'_\Phi$ admits a strictly positive density $q'(x)$ everywhere. On the support of $P$, we strictly have $\frac{p(x)}{p(x)+q'(x)} < 1$. By the strict convexity of the divergence:
\[
\JS(P \parallel G'_\Phi) = \frac{1}{2} \mathbb{E}_{P} \left[ \log \frac{2p(x)}{p(x)+q'(x)} \right] + \frac{1}{2} \mathbb{E}_{G'_\Phi} \left[ \log \frac{2q'(x)}{p(x)+q'(x)} \right] < \log 2.
\]
\end{proof}

Our second strategy, History Matching, aligns retrieved human trajectories with the task. We prove that this approach is asymptotically superior to the Raw Agent.

\begin{theorem}
\label{thm:history_matching}
Let $\Psi: \mathcal{X} \to \mathcal{S}$ be a feature mapping to a compact metric space $\mathcal{S}$. Let $P_\Psi$ be the human distribution and $G_{\Phi, \Psi} = \delta_{\mathbf{0}}$ be the Raw Agent. Let $\hat{P}_N$ be the empirical measure of History Matching. As $N \to \infty$:
\begin{equation}
    \lim_{N \to \infty} \sup_{D \in \mathcal{D}_L} \mathcal{L}_D(D; P_\Psi, \hat{P}_N) = -\log 4 < \sup_{D \in \mathcal{D}_L} \mathcal{L}_D(D; P_\Psi, G_{\Phi, \Psi}).
\end{equation}
\end{theorem}
\begin{proof}
Let $V(Q) = \sup_{D} \mathcal{L}_D(D; P_\Psi, Q)$. For the History Matching agent, using the inequality $|\sup f - \sup g| \le \sup |f-g|$:
\begin{align*}
|V(\hat{P}_N) - (-\log 4)| &= |V(\hat{P}_N) - V(P_\Psi)| \\
&\le \sup_{D \in \mathcal{D}_L} \left| \mathbb{E}_{x \sim \hat{P}_N}[\log(1 - D(x))] - \mathbb{E}_{x \sim P_\Psi}[\log(1 - D(x))] \right|.
\end{align*}
Since $h(x) = \log(1-D(x))$ is Lipschitz bounded by constant $K$, Kantorovich-Rubinstein duality implies the bound $K \cdot W_1(\hat{P}_N, P_\Psi)$, which converges to 0 almost surely.

As for the Raw Agent $G_{\Phi, \Psi} = \delta_{\mathbf{0}}$, consider the perturbation $D_\epsilon(x) = \sigma(2\epsilon \cdot \min(\|x\|, 1))$ around the trivial detector $D_0=0.5$. Using the Taylor expansion $\log(0.5+u) \approx -\log 2 + 2u$:
\begin{align*}
\mathcal{L}(D_\epsilon) &\approx \mathbb{E}_{P_\Psi}[-\log 2 + 2(D_\epsilon(x)-0.5)] + [-\log 2 - 2(D_\epsilon(\mathbf{0})-0.5)] \\
&= -\log 4 + 2\mathbb{E}_{x \sim P_\Psi}[D_\epsilon(x) - 0.5] \\
&\approx -\log 4 + \epsilon \mathbb{E}_{x \sim P_\Psi}[\min(\|x\|, 1)].
\end{align*}
Since $P_\Psi \neq \delta_{\mathbf{0}}$, the expectation is strictly positive. Thus, $\sup \mathcal{L} > -\log 4$.
\end{proof}

\section{Related Work}
\label{app:related_works}
\subsection{LMM-based GUI Agents}
The evolution of mobile automation has transitioned from rigid, rule-based scripts to autonomous agents empowered by Large Multimodal Models (LMMs). Early automation frameworks, such as Selenium and Appium, relied heavily on static XML view hierarchies and predefined coordinate scripts, making them brittle to UI updates and lacking semantic understanding. 

The emergence of LMMs~\cite{gpt4, gemini, llava} has catalyzed a paradigm shift. Recent works like AppAgent~\cite{appagent}, Mobile-Agent~\cite{mobile_agent,ye2508mobile}, CogAgent~\cite{cogagent}, and others~\cite{ma2024coco,li2025mobileuse,ma2024caution} utilize the strong visual perception and reasoning capabilities of models to interact with mobile interfaces in a manner akin to human users. These agents can interpret screenshots, reason about task goals, and execute actions through a general-purpose action space (e.g., tap, swipe). Subsequent research has further expanded these capabilities to web navigation~\cite{mind2web, yao2022webshop}, adaptive OS-level control~\cite{wu2025verios,cheng2025kairos}, and complex agentic information retrieval and deep search tasks~\cite{zhang2025agenticinformationretrieval,xi2025surveyllmbaseddeepsearch}. 

However, as highlighted by recent comprehensive surveys on agent evaluation~\cite{zhu2025evolutionaryperspectivesevaluationllmbased}, the primary objective of these existing works and benchmarks is heavily skewed towards maximizing \textbf{Task Success Rate (Utility)} and efficiency. This goal has driven the adoption of advanced optimization techniques, such as reinforcement learning and policy optimization~\cite{gu2025mobile,lu2025arpo,xumobilerl}, to refine agent decision-making. Consequently, the motion control modules of these agents are often implemented using deterministic algorithms. While effective for task completion, these efficient but unnatural kinematic patterns create a distinct behavioral gap compared to human users, leaving them vulnerable to detection.

\subsection{Adversarial Dynamics in Digital Ecosystems}
The widespread deployment of autonomous agents has precipitated a structural conflict within the digital economy. As noted by recent studies~\cite{lin2025superplatformsattackaiagents, allouah2025aiagentbuyingevaluation}, dominant digital platforms rely heavily on the attention economy, whereas agents are optimized for efficiency, often bypassing ads and promotional content.

This misalignment of incentives has triggered an adversarial dynamic. Existing research in this domain predominantly focuses on the axis of \textbf{Robustness} versus \textbf{Perturbation}~\cite{xi2025rise, wu2025dissectingadversarialrobustnessmultimodal, zhang2025agentsecuritybenchasb, xu2025advagentcontrollableblackboxredteaming}. For instance, recent works~\cite{gu2024agent,cui2023robustnesslargemultimodalmodels,lin2025superplatformsattackaiagents,dong2023robust} have demonstrated how platforms can launch adversarial attacks to disrupt an agent's visual grounding. These threats range from environmental injections~\cite{liao2025eiaenvironmentalinjectionattack, chen2025evaluatingrobustnessmultimodalagents, chen2025obviousinvisiblethreatllmpowered, zhang2025attackingvisionlanguagecomputeragents} and visual adversaries~\cite{fang2024clipguidedgenerativenetworkstransferable, zhang2025qavaqueryagnosticvisualattack, de2024exploring} to more insidious backdoor and jailbreak triggers~\cite{wang2024badagentinsertingactivatingbackdoor, yang2024watchagentsinvestigatingbackdoor, weng-etal-2025-foot}. In response, strategies have been proposed to fine-tune agents to resist such visual and structural perturbations.

However, this perspective addresses only the availability of agents, overlooking the prerequisite of invisibility. Before deploying complex adversarial perturbations, platforms are incentivized to first deploy passive detection mechanisms, utilizing techniques such as behavioral biometrics and fingerprinting, to filter out non-human actors to avoid degrading the experience for real users. Our work shifts the focus from robustness against functional attacks to \textbf{survivability against behavioral detection}, framing the interaction as a ``Turing Test on Screen.''

\subsection{Bot Detection and Behavioral Biometrics}
The detection of automated agents has long been a critical component of digital security. Traditional bot detection literature~\cite{mahfouz2017survey,vastel2018fp,laperdrix2020browser} focuses primarily on identifying rigid scripts associated with web crawlers. These methods often rely on analyzing deterministic patterns, such as fixed inter-arrival times, repetition of identical coordinate sequences, or inconsistencies in browser fingerprints.

In the mobile domain, detection often leverages \textbf{Behavioral Biometrics}, which treats touch dynamics as a unique signature for user identity. Research in this field~\cite{zaidi2021touch, frank2012touchalytics, alrawili2024comprehensivesurveybiometricuser, Feng2012ContinuousMA, Kroeze2016UserAB, Shen2024IncreAuthIB, alrawili2024comprehensivesurveybiometricuser} utilizes features such as touch pressure, contact area, finger velocity, and trajectory curvature to distinguish between different human users. While some works have extended these principles to mouse dynamics~\cite{khan2024mouse} or game avatar trajectories~\cite{pao2010game}, the primary focus remains on user verification. Specifically, studies have explored the resilience of these systems against simple replay attacks~\cite{ForgResTouchAuth, zaidi2021touch} and even robotic attacks on touchscreens~\cite{serwadda2016toward}, with recent advancements proposing GAN-based frameworks to enhance robustness~\cite{9893211}.

Despite this rich history, a critical gap exists in the era of LMM-based agents. Current detection paradigms operate under the assumption that bots are either dumb scripts with zero variance or replay attacks with perfect repetition. LMM-based agents, however, represent a new class of adversary: they possess stochastic decision-making capabilities but often exhibit mechanical execution. There is currently a lack of systematic research addressing the detectability of these advanced agents, which sit in the uncanny valley between rigid bots and natural humans.
\section{Further Discussion}
\label{app:discussion}

\subsection{The Robustness of Detection Baselines.}

A potential concern regarding our benchmark is the choice of feature-based classifiers (SVM and XGBoost) over deep sequence models like LSTMs~\cite{graves2012long} or Transformers~\cite{vaswani2017attention}. We argue that this setup is both scientifically rigorous and practically representative. Empirically, we observe a saturation of Detection effect: raw LMM agents exhibit such high levels of mechanistic regularity (e.g., near-zero variance in velocity and perfect linearity) that even shallow statistical models achieve near-perfect accuracy ($>99\%$), rendering more complex neural architectures redundant for current identification tasks. Furthermore, the focus on interpretable features allows us to pinpoint the specific behavioral patterns of agents, which is more conducive to the iterative refinement of humanization strategies. From a theoretical standpoint (see Theorem~\ref{thm:history_matching}), our History Matching approach aims to minimize the J-S divergence between agent and human distributions. As this divergence approaches zero, the performance of any classifier, regardless of its architecture, is bounded by random guessing. Thus, AHB provides a foundational metric that addresses the root of detectability across varying levels of detector complexity.

\subsection{Delineating the Scope From Motion Dynamics to Physical Sensors}

A natural extension of the Turing Test on Screen would involve auxiliary sensor data, such as gyroscope and accelerometer readings, which real-world anti-abuse systems often utilize to verify physical device movement. In this study, we intentionally constrain our primary focus to the humanization of MotionEvents . We justify this prioritization based on two factors. First, touch dynamics constitute the first line of defense. As demonstrated in our empirical analysis, current agents fail so significantly at the interaction layer that platforms can effectively flag them without invoking high-power sensor monitoring. Second, high-fidelity sensor simulation presents a distinct system-level challenge, often requiring kernel-level signal injection or physical robotics, which falls outside the scope of algorithmic agent policy design. Future work will explore the cross-modal alignment between virtual touch interactions and physical sensor signals, aiming to construct a holistic humanization framework that synchronizes cinematic screen movements with believable inertial noise.

\subsection{The Imitability-Utility Pareto Frontier.}

A critical observation in our benchmark is the significant utility degradation associated with certain humanization strategies, most notably the fake action injection in complex tasks like Trip Planning in Table~\ref{tab:main table}. We argue that these results empirically expose a fundamental \textbf{Pareto Frontier} between behavioral imitability and task utility. The collapse of success rates (from 0.75 to 0.15) in the Online + Fake configuration reveals that naive noise injection, while effective at obfuscating statistical signatures, frequently violates the underlying logic of the GUI, leading to unintended side effects such as accidental navigation or session timeouts. In contrast, our History Matching strategy maintains a much more favorable balance, achieving high imitability with significantly less utility loss. By documenting these failures, AHB provides a crucial cautionary tale for the community: humanization must be context-aware rather than purely stochastic. These findings justify the need for more advanced Guard Agents (as discussed in Section~\ref{sec:the hardest features}) that can generate human-like noise without logical interference, setting a new research agenda for the next generation of GUI agents~\cite{liu2026positionrealbarrierllm}. 

\subsection{Broader Impact and Ethical Considerations.}

The introduction of the AHB benchmark raises important ethical questions regarding the potential for bypassing anti-bot measures. However, we contend that the release of this research serves the long-term interests of platform security through a Red Teaming philosophy. Historically, malicious actors (e.g., click farms) have operated with proprietary, non-transparent evasion techniques, leaving defenders in a reactive stance. By formalizing the behavioral gaps of LMM agents and providing standardized detection baselines, our work empowers platform defenders to identify and mitigate sophisticated bot behaviors more effectively. Furthermore, our primary motivation is to \textbf{safeguard User Agency}. As illustrated by the case in Appendix~\ref{app:doubao_case}, current all-or-nothing defense mechanisms often inadvertently penalize legitimate users employing AI assistants for accessibility or productivity. Our research advocates for a shift toward more nuanced, behavioral-aware authentication that can distinguish between constructive automation and malicious exploitation. We will release our dataset and evaluation harness under a restrictive research license to ensure they are used primarily for advancing the robustness of defense systems and the development of ethical AI assistants.

\newpage
\section{System Prompt for Mobile GUI Agent}
\label{app:system_prompt}

The following is the specific system prompt template used for the \textbf{MOBILE\_USE\_DOUBAO} configuration in UI-TARS. This prompt defines the action space, output constraints, and task-specific instructions for the agent.

\subsection{UI-TARS}

\begin{lstlisting}
MOBILE_USE_DOUBAO = """You are a GUI agent. You are given a task and your action history, with screenshots. You need to perform the next action to complete the task. 
## Output Format
```
Thought: ...
Action: ...
```
## Action Space

click(point='<point>x1 y1</point>')
long_press(point='<point>x1 y1</point>')
type(content='') #If you want to submit your input, use "\\n" at the end of `content`.
scroll(point='<point>x1 y1</point>', direction='down or up or right or left')
open_app(app_name=\'\')
drag(start_point='<point>x1 y1</point>', end_point='<point>x2 y2</point>')
press_home()
press_back()
finished(content='xxx') # Use escape characters \\', \\", and \\n in content part to ensure we can parse the content in normal python string format.


## Note
- Use {language} in `Thought` part.
- Write a small plan and finally summarize your next action (with its target element) in one sentence in `Thought` part.

## User Instruction
{instruction}
"""
Search for flights from Beijing to Shenzhen on the 16th of a specific month, filter by departure time between 12:00 and 18:00, specify economy class, select one flight, and view detailed refund and change information.
\end{lstlisting}

\subsection{Mobile-Agent-E }
\label{app:mobile_agent_e}

This section details the hierarchical prompt templates used in Mobile-Agent-E. The templates are dynamically populated with environmental metadata (e.g., coordinates, keyboard status) and historical context.

\subsection{Action Perception Prompt}
The following template is used to generate the next operational step based on the current screenshot and history.

\begin{lstlisting}[basicstyle=\footnotesize\ttfamily, breaklines=true, frame=single, backgroundcolor=\color{gray!5}]
### Background ###
This image is a phone screenshot. Its width is {width} pixels and its height is {height} pixels. The user's instruction is: {instruction}.

### Screenshot information ###
In order to help you better perceive the content... [Coordinates]; [Content]
{clickable_infos}
Please note that this information is not necessarily accurate.

### Keyboard status ###
{The keyboard has (not) been activated...}

### History operations ###
Step-1: [Operation: {thought}; Action: {action}]

### Progress ###
Completed contents: {completed_content}

### Response requirements ###
You must choose one of the six actions below:
1. Open app (app name)
2. Tap (x, y)
3. Swipe (x1, y1), (x2, y2)
4. Type (text) / Unable to Type
5. Home
6. Stop

### Output format ###
### Thought ###
### Action ###
### Operation ###
\end{lstlisting}

\subsection{Action Reflection Prompt}
Used after an operation to verify if the result meets the expected thought.

\begin{lstlisting}[basicstyle=\footnotesize\ttfamily, breaklines=true, frame=single, backgroundcolor=\color{gray!5}]
### Before the current operation ###
Screenshot info & Keyboard status...

### After the current operation ###
Screenshot info & Keyboard status...

### Current operation ###
Instruction: {instruction}
Operation thought: {summary}
Operation action: {action}

### Response requirements ###
A: The result meets my expectation.
B: Results in a wrong page (need to return).
C: Produces no changes.

### Output format ###
### Thought ###
### Answer ###
\end{lstlisting}

\subsection{Memory \& Process Update Prompts}
Templates for maintaining long-term knowledge and tracking task completion progress.

\begin{lstlisting}[basicstyle=\footnotesize\ttfamily, breaklines=true, frame=single, backgroundcolor=\color{gray!5}]
% Memory Prompt Segment
### Response requirements ###
Is there any content closely related to ### Important content ### on the current page? 
### Output format ###
### Important content ### {Content or None}

% Process Prompt Segment
### Progress thinking ###
Completed contents: {completed_content}
### Response requirements ###
Update the "Completed contents". Don't output the purpose, just summarize what has been actually completed.
### Output format ###
### Completed contents ###
\end{lstlisting}

\subsection{Agent-CPM}
\label{app:agent_cpm_prompt}

The system prompt for Agent-CPM utilizes a schema-driven approach to ensure structured output and precise coordinate-based grounding. The integrated JSON schema constrains the model to output valid operational logic for Android GUI environments.

\begin{lstlisting}[
    basicstyle=\small\ttfamily,
    breaklines=true,
    frame=single,
    backgroundcolor=\color{gray!5},
    caption={Synthesized System Prompt for Agent-CPM (CPM-GUI)},
    label={lst:agent_cpm_full}
]
# Role
You are an agent familiar with Android system touchscreen GUI operations. You will analyze the GUI elements and layout of the current interface based on user questions and generate corresponding operations.

# Task
Based on the user's question and the input screenshot of the current screen, output the next operation.

# Rule
- Output in compact JSON format.
- Output operations must follow the Schema constraints.

# Schema
{
  "type": "object",
  "description": "Execute action and determine current task status",
  "additionalProperties": false,
  "properties": {
    "thought": {
      "type": "string",
      "description": "The agent's thinking process"
    },
    "POINT": {
      "$ref": "#/$defs/Location",
      "description": "Click on a specified location on the screen"
    },
    "to": {
      "description": "Movement, composite gesture parameters",
      "oneOf": [
        {
          "enum": ["up", "down", "left", "right"],
          "description": "From the current point (POINT), perform a swipe gesture in directions: up, down, left, right"
        },
        {
          "$ref": "#/$defs/Location",
          "description": "Move to a certain location"
        }
      ]
    },
    "duration": {
      "type": "integer",
      "description": "Execution time or wait time for the action, in milliseconds",
      "minimum": 0,
      "default": 200
    },
    "PRESS": {
      "type": "string",
      "description": "Trigger special keys: HOME (go to home page), BACK (return button), ENTER (return key)",
      "enum": ["HOME", "BACK", "ENTER"]
    },
    "TYPE": {
      "type": "string",
      "description": "Input text"
    },
    "STATUS": {
      "type": "string",
      "description": "Current status of the task. Special cases: satisfied (no action needed), impossible (task cannot be completed), interrupt (task interrupted), need_feedback (user feedback required)",
      "enum": ["continue", "finish", "satisfied", "impossible", "interrupt", "need_feedback"],
      "default": "continue"
    }
  },
  "$defs": {
    "Location": {
      "type": "array",
      "description": "Coordinates relative to the top-left corner of the screen, scaled to 0-1000 based on width/height ratio. [x, y]",
      "items": { "type": "integer", "minimum": 0, "maximum": 1000 },
      "minItems": 2, "maxItems": 2
    }
  }
}
\end{lstlisting}

\subsection{Open-AutoGLM System Prompt}
\label{app:open_autoglm_prompt}

The system prompt for Open-AutoGLM is designed to guide the model through a loop of screen perception and action execution. The prompt resides in \texttt{phone\_agent/config/prompts\_en.py} (or \texttt{prompts\_zh.py} for Chinese).

\begin{lstlisting}[
    basicstyle=\small\ttfamily,
    breaklines=true,
    frame=single,
    backgroundcolor=\color{gray!5},
    caption={System Prompt for Open-AutoGLM},
    label={lst:open_autoglm}
]
# Role
You are a professional mobile phone operation assistant. You need to analyze the current screenshot and task history to help users complete their requests on an Android device.

# Task
Your goal is: {instruction}
Current Screen Resolution: {width} x {height}

# Guidelines
1. Observe the current screenshot carefully.
2. Consider the previous actions and the progress made so far.
3. Determine the next logical step. If the task is completed, use the "stop" action.
4. All coordinates must be normalized to a range of 0 to 1000.

# Action Space
- click(x, y): Tap the screen at normalized coordinates (x, y).
- swipe(x1, y1, x2, y2): Swipe from (x1, y1) to (x2, y2).
- type(text): Type the specified text into the focused input field.
- key(name): Press system keys like 'HOME', 'BACK', or 'MENU'.
- wait(): Wait for the screen to update or an app to load.
- stop(summary): Finalize the task and provide a summary of what was achieved.

# Output Format
Thought: Your reasoning for the next step.
Action: The function call representing your action.
\end{lstlisting}
\newpage

\section{Experiment Results}
\label{app:experiment results}

\begin{figure}[ht]
  \vspace{-10pt}
  \begin{center}
    \centerline{\includegraphics[width=0.9\textwidth]{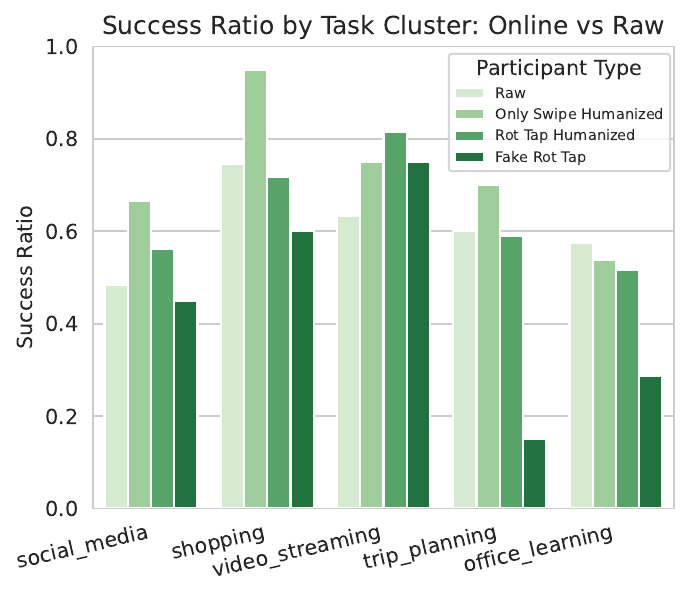}}
    
    \vspace{-10pt}
    \caption{Impact of Online Humanization on Task Utility. This chart compares the success rates of raw agents (light green) versus those employing various humanization strategies (darker and darker green). }
    \label{fig:success rate comparison}
  \end{center}
    
  \vspace{-30pt}
\end{figure}

\begin{table*}[ht]
\centering
\caption{\textbf{Merged Results: Cluster 0 vs Cluster 1 vs Cluster 2}}
\label{tab:merged_0_1_2_no_unh}

\tiny
\setlength{\tabcolsep}{2pt} 

\resizebox{\linewidth}{!}{%
\begin{tabular}{l cccc cccc cccc}
\toprule
\multirow{2}{*}{\textbf{Metric}} & 
\multicolumn{4}{c}{\textbf{Cluster 0: Social Media}} & 
\multicolumn{4}{c}{\textbf{Cluster 1: Shopping}} &
\multicolumn{4}{c}{\textbf{Cluster 2: VideoStreaming}} \\
\cmidrule(lr){2-5} \cmidrule(lr){6-9} \cmidrule(lr){10-13}

& \textbf{RAW} & \textbf{On.RM} & \textbf{Off.RM} & \textbf{BS} 
& \textbf{RAW} & \textbf{On.RM} & \textbf{Off.RM} & \textbf{BS}
& \textbf{RAW} & \textbf{On.RM} & \textbf{Off.RM} & \textbf{BS} \\ 
\midrule

maxDev & 0.9969 & 0.5515 & 0.6186 & 0.7556 & 0.9899 & 0.6231 & 0.7710 & 0.5379 & 0.9929 & 0.6636 & 0.7146 & 0.5026 \\
meanResultantLength & 0.9878 & 0.6818 & 0.6286 & 0.6979 & 0.9982 & 0.8556 & 0.8183 & 0.9336 & 1.0000 & 0.8700 & 0.8476 & 0.9003 \\
ratio\_end\_to\_len & 0.9878 & 0.6451 & 0.5798 & 0.5826 & 0.9980 & 0.8556 & 0.8537 & 0.9002 & 1.0000 & 0.8447 & 0.8338 & 0.8476 \\
duration & 0.8583 & 0.6907 & 0.5470 & 0.8507 & 0.9209 & 0.7554 & 0.8351 & 0.9230 & 0.9175 & 0.7517 & 0.7560 & 0.9147 \\
a20 & 0.8355 & 0.8286 & 0.7190 & 0.7686 & 0.9632 & 0.9249 & 0.8780 & 0.9309 & 0.9679 & 0.9494 & 0.9306 & 0.9390 \\
acc\_first5pct & 0.8244 & 0.5897 & 0.5532 & 0.8093 & 0.5155 & 0.7407 & 0.8301 & 0.5430 & 0.5652 & 0.7551 & 0.7782 & 0.5242 \\
a80 & 0.8154 & 0.6575 & 0.6205 & 0.6560 & 0.5708 & 0.8516 & 0.8455 & 0.7229 & 0.5754 & 0.8577 & 0.8641 & 0.6784 \\
dev80 & 0.7645 & 0.5310 & 0.6445 & 0.5476 & 0.6553 & 0.6863 & 0.7140 & 0.7685 & 0.6173 & 0.6668 & 0.7105 & 0.7239 \\
dev20 & 0.7634 & 0.5038 & 0.5560 & 0.5111 & 0.8765 & 0.7317 & 0.7315 & 0.8176 & 0.8656 & 0.7076 & 0.7091 & 0.7907 \\
dev50 & 0.7055 & 0.5510 & 0.6416 & 0.5251 & 0.7362 & 0.6777 & 0.7027 & 0.7994 & 0.7063 & 0.6879 & 0.7049 & 0.7696 \\
v80 & 0.6996 & 0.5188 & 0.6301 & 0.7026 & 0.5980 & 0.7039 & 0.7898 & 0.6466 & 0.6229 & 0.7007 & 0.7404 & 0.6544 \\
avgDirection & 0.6763 & 0.5515 & 0.5900 & 0.5872 & 0.7140 & 0.7380 & 0.7702 & 0.7660 & 0.7355 & 0.7528 & 0.7560 & 0.7674 \\
direction & 0.6734 & 0.5470 & 0.6734 & 0.5758 & 0.7140 & 0.7157 & 0.7140 & 0.7635 & 0.7278 & 0.7336 & 0.7278 & 0.7595 \\
startY & 0.6581 & 0.7472 & 0.6581 & 0.6581 & 0.7965 & 0.8228 & 0.7965 & 0.7965 & 0.8656 & 0.8907 & 0.8656 & 0.8656 \\
speed & 0.6531 & 0.5000 & 0.6100 & 0.6329 & 0.8183 & 0.7317 & 0.7921 & 0.8210 & 0.7429 & 0.6922 & 0.7226 & 0.7380 \\
startX & 0.6531 & 0.5310 & 0.6531 & 0.5730 & 0.7844 & 0.7872 & 0.7844 & 0.8223 & 0.7976 & 0.8248 & 0.7976 & 0.8380 \\
a50 & 0.6447 & 0.7789 & 0.6953 & 0.6501 & 0.9117 & 0.9224 & 0.8513 & 0.9002 & 0.9164 & 0.9321 & 0.9275 & 0.9108 \\
displacement & 0.6416 & 0.5154 & 0.6416 & 0.6387 & 0.8262 & 0.5907 & 0.8262 & 0.8249 & 0.6154 & 0.5648 & 0.6154 & 0.6301 \\
v50 & 0.6329 & 0.5407 & 0.6243 & 0.6387 & 0.8156 & 0.7468 & 0.8045 & 0.8169 & 0.7278 & 0.7410 & 0.7465 & 0.7278 \\
endX & 0.6243 & 0.5262 & 0.6243 & 0.5251 & 0.7950 & 0.7989 & 0.7950 & 0.8269 & 0.8918 & 0.8700 & 0.8918 & 0.8808 \\
endY & 0.6178 & 0.7279 & 0.6178 & 0.6232 & 0.7567 & 0.8567 & 0.7567 & 0.7567 & 0.8515 & 0.9156 & 0.8515 & 0.8515 \\
v20 & 0.6014 & 0.5479 & 0.5843 & 0.6014 & 0.8810 & 0.7678 & 0.7883 & 0.8775 & 0.8191 & 0.7696 & 0.7824 & 0.8110 \\
length & 0.5617 & 0.5271 & 0.5560 & 0.5588 & 0.8135 & 0.6550 & 0.7935 & 0.7898 & 0.5373 & 0.5221 & 0.5732 & 0.5602 \\
v\_last3\_median & 0.5419 & 0.6637 & 0.6313 & 0.5560 & 0.8713 & 0.8699 & 0.8009 & 0.8681 & 0.8287 & 0.8776 & 0.8627 & 0.8262 \\
svm\_accuracy & 0.9817 & 0.8750 & 0.9633 & 0.9633 & 0.9887 & 0.9593 & 0.9323 & 0.9774 & 0.9850 & 0.9502 & 0.9300 & 0.9650 \\
xgb\_accuracy & 1.0000 & 0.9773 & 0.9450 & 0.9817 & 1.0000 & 0.9889 & 0.9925 & 0.9925 & 1.0000 & 0.9950 & 0.9850 & 0.9850 \\
\bottomrule
\end{tabular}
}

\vspace{0.2cm}
\footnotesize
\textbf{Legend:} \textit{On.RM}: Online Rotation \& Match; \textit{Off.RM}: Offline Rotation \& Match; \textit{BS}: B-Spline.
\end{table*}

\begin{table*}[h]
\centering
\caption{\textbf{Merged Results: Cluster 3 vs Cluster 4}}
\label{tab:merged_3_4_no_unh}

\resizebox{\linewidth}{!}{%
\begin{tabular}{l cccc cccc}
\toprule
\multirow{2}{*}{\textbf{Metric}} & 
\multicolumn{4}{c}{\textbf{Cluster 3: Trip Planning}} & 
\multicolumn{4}{c}{\textbf{Cluster 4: Office \& Learning}} \\
\cmidrule(lr){2-5} \cmidrule(lr){6-9}

& \textbf{RAW} & \textbf{On.RM} & \textbf{Off.RM} & \textbf{BS} 
& \textbf{RAW} & \textbf{On.RM} & \textbf{Off.RM} & \textbf{BS} \\ 
\midrule

maxDev & 0.9895 & 0.7188 & 0.6508 & 0.5629 & 1.0000 & 0.6347 & 0.5841 & 0.8178 \\
meanResultantLength & 0.9940 & 0.7806 & 0.7490 & 0.8479 & 1.0000 & 0.5181 & 0.5625 & 0.7720 \\
ratio\_end\_to\_len & 0.9984 & 0.7573 & 0.7266 & 0.7830 & 1.0000 & 0.5294 & 0.5702 & 0.7366 \\
duration & 0.8601 & 0.7434 & 0.6696 & 0.8648 & 0.6159 & 0.5944 & 0.5052 & 0.6372 \\
a20 & 0.9376 & 0.8153 & 0.8251 & 0.8855 & 0.8521 & 0.7204 & 0.7720 & 0.7543 \\
acc\_first5pct & 0.6396 & 0.7445 & 0.7132 & 0.6082 & 0.7580 & 0.5662 & 0.6398 & 0.7283 \\
a80 & 0.5940 & 0.7913 & 0.7801 & 0.5733 & 0.7824 & 0.5616 & 0.5935 & 0.6772 \\
dev80 & 0.5086 & 0.5949 & 0.5748 & 0.6584 & 0.7526 & 0.5334 & 0.5314 & 0.5521 \\
dev20 & 0.8552 & 0.6566 & 0.6365 & 0.7437 & 0.8247 & 0.5922 & 0.5419 & 0.6115 \\
dev50 & 0.5909 & 0.5918 & 0.5895 & 0.7030 & 0.7921 & 0.5685 & 0.5236 & 0.5935 \\
v80 & 0.5017 & 0.7071 & 0.6818 & 0.5151 & 0.5815 & 0.5616 & 0.5806 & 0.6186 \\
avgDirection & 0.5234 & 0.6332 & 0.6204 & 0.5880 & 0.6106 & 0.6050 & 0.5288 & 0.5921 \\
direction & 0.5151 & 0.5395 & 0.5151 & 0.5822 & 0.6372 & 0.6762 & 0.6372 & 0.6133 \\
startY & 0.5807 & 0.6749 & 0.5807 & 0.5851 & 0.5157 & 0.6262 & 0.5157 & 0.5157 \\
speed & 0.6204 & 0.7106 & 0.6806 & 0.6231 & 0.6475 & 0.5246 & 0.5651 & 0.6424 \\
startX & 0.6721 & 0.6388 & 0.6721 & 0.6842 & 0.5261 & 0.5226 & 0.5261 & 0.5261 \\
a50 & 0.8471 & 0.8054 & 0.8421 & 0.8294 & 0.6986 & 0.6980 & 0.7366 & 0.6858 \\
displacement & 0.7211 & 0.7894 & 0.7211 & 0.7121 & 0.6532 & 0.7782 & 0.6532 & 0.5183 \\
v50 & 0.6176 & 0.7211 & 0.6854 & 0.6404 & 0.5909 & 0.5090 & 0.5702 & 0.5599 \\
endX & 0.6830 & 0.6471 & 0.6830 & 0.6949 & 0.5183 & 0.5134 & 0.5183 & 0.5183 \\
endY & 0.5880 & 0.6888 & 0.5880 & 0.5938 & 0.5104 & 0.6050 & 0.5104 & 0.5052 \\
v20 & 0.7521 & 0.7268 & 0.6770 & 0.7395 & 0.7265 & 0.5067 & 0.5728 & 0.7011 \\
length & 0.6571 & 0.7477 & 0.6647 & 0.6659 & 0.6799 & 0.6136 & 0.5683 & 0.5498 \\
v\_last3\_median & 0.7653 & 0.7257 & 0.7458 & 0.7320 & 0.6654 & 0.6132 & 0.6552 & 0.6577 \\
svm\_accuracy & 0.9817 & 0.9479 & 0.8995 & 0.9726 & 0.9826 & 0.9265 & 0.9391 & 0.9739 \\
xgb\_accuracy & 0.9954 & 0.9905 & 0.9863 & 0.9909 & 1.0000 & 0.9926 & 0.9739 & 0.9913 \\
\bottomrule
\end{tabular}
}

\vspace{0.2cm}
\footnotesize
\textbf{Legend:} \textit{On.RM}: Online Rotation \& Match; \textit{Off.RM}: Offline Rotation \& Match; \textit{BS}: B-Spline.
\end{table*}

\begin{figure}[ht]
  \begin{center}
    \centerline{\includegraphics[width=0.9\textwidth]{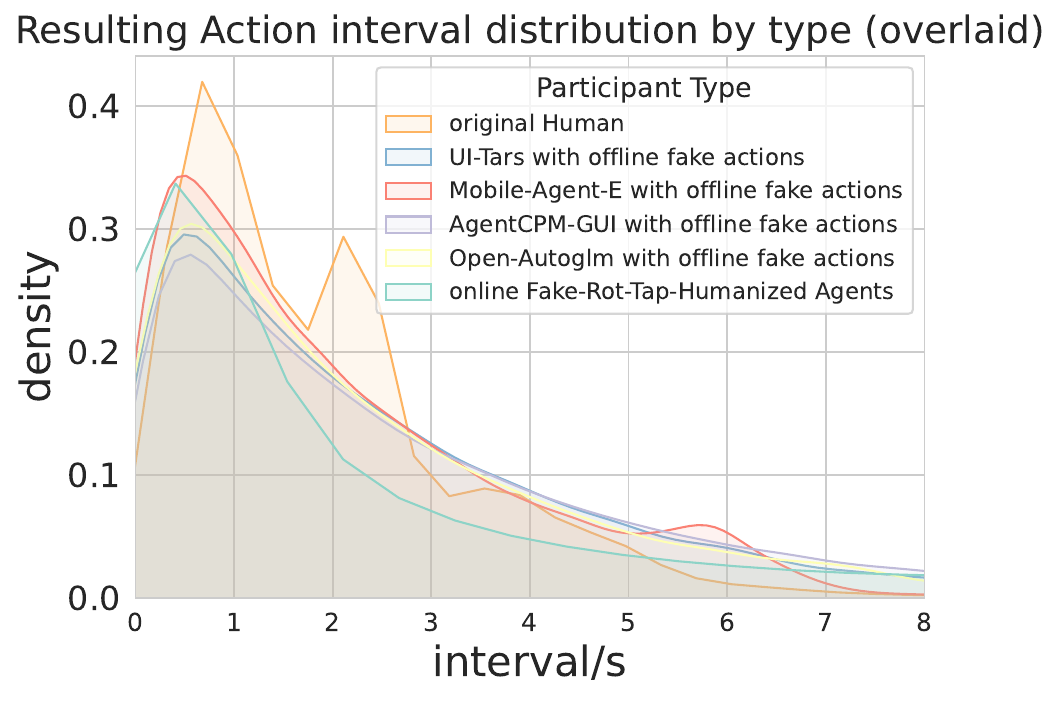}}
    \vspace{-10pt}
    \caption{Validation of interval mimicry using fake actions. The figure compares the normalized action interval distributions of the original human dataset (red) against four GUI agents (UI-Tars, Mobile-Agent-E, AgentCPM-GUI, Open-AutoGLM) augmented with offline fake actions and combined agents augmented with online fake actions. The significant overlap indicates that the proposed method successfully replicates the temporal distribution patterns of human behavior.}
    \label{fig:fake action interval}
  \end{center}
  \vspace{-10pt}
\end{figure}

\begin{figure}[ht]
  
  \vspace{-10pt}
  \begin{center}
    \centerline{\includegraphics[width=0.9\textwidth]{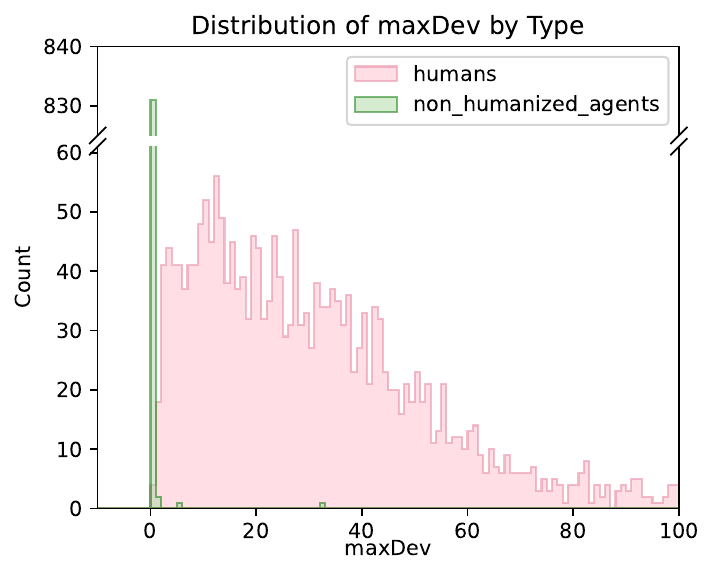}}
    \vspace{-10pt}
    \caption{Distribution analysis of trajectory deviation. We compare the distribution of the $maxDev$ feature for humans versus non-humanized agents. While human data follows a wide distribution reflecting natural motor variability, agent data is concentrated in a singular impulse near zero, confirming the linearity of algorithmically generated paths without humanization.}
    \label{fig:Max Deviation}
  \end{center}
  
  \vspace{-10pt}
\end{figure}

\end{document}